\def\year{2020}\relax
\documentclass[letterpaper]{article} 
\usepackage{aaai20}  
\usepackage{times}  
\usepackage{helvet} 
\usepackage{courier}  
\usepackage[hyphens]{url}  
\usepackage{graphicx} 
\usepackage{graphicx}  
\usepackage{amssymb}
\usepackage{soul}
\usepackage{url}
\usepackage[hidelinks]{hyperref}
\usepackage[utf8]{inputenc}
\usepackage[small]{caption}
\usepackage{algorithm}
\usepackage{algorithmic}
\usepackage{threeparttable}
\usepackage{amsmath}
\usepackage{booktabs}
\usepackage{verbatim}
\usepackage{subfigure}
\usepackage{multirow}
\usepackage{amsthm}     

\usepackage{xcolor}
\usepackage{graphicx}

\newcommand{\m}[1]{\mathbf{#1}}

\usepackage[utf8]{inputenc}
\usepackage[english]{babel}

\newtheorem{theorem}{Theorem}

\newtheorem{lemma}{Lemma}

\newcommand{\prox}{\mathrm{Prox}}
\newcommand{\x}{\times}
\title{Few Shot Network Compression via Cross Distillation}
\author{Haoli Bai,\textsuperscript{\rm 1} Jiaxiang Wu,\textsuperscript{\rm 2}, Irwin King,\textsuperscript{\rm 1} Michael Lyu\textsuperscript{\rm 1}\\   
\textsuperscript{\rm 1}The Chinese University of Hong Kong \ \ \textsuperscript{\rm 2}Tencent AI Lab\\
\{hlbai, king, lyu\}@cse.cuhk.edu.hk, jonathanwu@tencent.com 
}

\begin{document}

\maketitle

\begin{abstract}
Model compression has been widely adopted to obtain light-weighted deep neural networks. Most prevalent methods, however, require fine-tuning with sufficient training data to ensure accuracy, which could be challenged by privacy and security issues. 
As a compromise between privacy and performance, in this paper we investigate few shot network compression: given few samples per class, how can we effectively compress the network with negligible performance drop?
The core challenge of few shot network compression lies in high estimation errors from the original network during inference, since the compressed network can easily over-fits on the few training instances. The estimation errors could propagate and accumulate layer-wisely and finally deteriorate the network output.
To address the problem, we propose cross distillation, a novel layer-wise knowledge distillation approach. By interweaving hidden layers of teacher and student network, layer-wisely accumulated estimation errors can be effectively reduced.
The proposed method offers a general framework compatible with prevalent network compression techniques such as pruning.
Extensive experiments on benchmark datasets demonstrate that cross distillation can significantly improve the student network's accuracy when only a few training instances are available.
\end{abstract}

\section{Introduction}

Deep neural networks~(DNNs) have achieved remarkable success in a wide range of applications, however, they suffer from substantial computation and energy cost. 
In order to obtain light-weighted DNNs, network compression techniques have been widely developed in recent years, including  network pruning~\cite{he2017channelprunning,luo2017thinet,wen2019structured}, quantization~\cite{han2016deepcompression,wu2016quantized,wu2018pocketflow,li2020rtn} and knowledge distillation~\cite{hinton2015distilling,romero2014fitnets}.

Despite the success of previous efforts, a majority of them rely on the whole training data to reboot the compressed models,
which could suffer from security and privacy issues. 
For instance, to provide a general service of network compression, the reliance on the training data may result in data leakage for customers.

To take care of security issues in network compression, some recent works~\cite{chen2019data,bhardwaj2019dream,lopes2017data} motivate from knowledge distillation~\cite{hinton2015distilling,romero2014fitnets}, and propose data-free fine-tuning by constructing pseudo inputs from the pre-trained teacher network. However, these methods highly rely on the quality of the pseudo inputs and are therefore limited to small-scale problems.

In order to obtain scalable network compression algorithms, a compromise between privacy and performance is to compress the network with few shot training instances, e.g., 1-shot for one training instances per class. 
Prevalent works~\cite{li2019knowledge,chen2019deep} along this line extend knowledge distillation by minimizing layer-wise estimation errors (e.g., Euclidean distances) between the teacher and student network. The success of these approaches largely comes from the layer-wise supervision from the teacher network.
Nevertheless, a key challenge in few shot network compression is rarely investigated in previous efforts: as there are few shot training samples available, the student network tend to over-fit on the training set and consequently suffer from high estimation errors from the teacher network during inference. Moreover, the estimation errors could propagate and accumulate layer-wisely~\cite{dong2017learning} and finally deteriorate the student network.

To deal with the above challenge, we proceed along with few shot network compression and propose cross distillation, a novel layer-wise knowledge distillation approach. Cross distillation can effectively reduce the layer-wisely accumulated errors in the few shot setting, leading to a more powerful and generalizable student network.
Specifically, to correct the errors accumulated in previous layers of the student network, we direct the teacher's hidden layers to the student network, which is called correction.
Meanwhile, to make the teacher aware of the errors accumulated on the student network, we reverse the strategy by directing the student's hidden layers to the teacher network. With error-aware supervision from the teacher, the student can better mimic the teacher's behavior, which is called imitation.
The correction and imitation compensate each other, and to find a proper trade-off, we propose to take convex combinations between either loss functions of the two procedures, or hidden layers of the two networks. 
To better understand the proposed method, we also give some theoretical analysis on how convex combination of the two loss functions manipulates the layer-wisely propagated errors, and why cross distillation is capable of improving the student network.
Our proposed method provides a universal framework to assist prevalent network compression techniques such as pruning~\cite{he2017channelprunning}.

Extensive experiments and ablation studies are conducted on popular network architectures and benchmark datasets, and the results demonstrate that our proposed method can effectively reduce the estimation errors and improve the compressed model in the few shot setting, outperforming a number of competitive baselines.

\section{Related Work}

While most previous efforts on network compression rely on abundant training data for fine-tuning the compressed network, there is a recent trend on investigating security and privacy issues for network compression.
These methods can be generally categorized into data-free methods and few-shot methods.

To perform data-free network compression, a simple way is to directly apply quantization~\cite{banner2018scalable} or low-rank factorization~\cite{zhang2015accelerating,ye2018learning} on network parameters, which usually degrade the network significantly when the compression rate is high.
Recent efforts motivate from knowledge distillation~\cite{hinton2015distilling,romero2014fitnets}, which constructs pseudo inputs from the pre-trained teacher network based on its parameters~\cite{nayak2019zero}, feature map statistics~\cite{lopes2017data,bhardwaj2019dream}, or an independently trained generative model~\cite{chen2019data} to simulate the distribution of the original training set. However, the generation of high-quality pseudo inputs could be challenging and expensive, especially on large-scale problems.

The other line of research considers network compression with few-shot training samples, which is a compromise between privacy and performance.
To fully take advantage of the training data, a number of existing works~\cite{he2017channelprunning,luo2017thinet,li2019knowledge,chen2019deep} extend knowledge distillation by layer-wisely minimizing the Euclidean distances between the teacher network and the student network. 
The layer-wise training is usually data-efficient as the student network receives layer-wise supervision from the teacher and there are fewer parameters to optimize comparing to back-propagation training of the entire student network~\cite{romero2014fitnets}.
Aside from layer-wise regression, recently data from different but related domains are also utilized as auxiliary information to assist the pruning on the target domain~\cite{chen2019cooperative}.
Unlike data free compression techniques, few shot network compression can significantly improve the performance of the compressed network with only limited training instances, which is potentially helpful for large-scale real-world problems.

Our proposed cross distillation proceeds along the line of few shot network compression. As an extension of previous layer-wise regression methods, we pay extra attention to the reduction of estimation errors during inference, which are usually large as a result of over-fitting on few shot training instances.
We remark that similar ideas of cross connection between two networks are also previously explored in multi-task learning~\cite{gupta2016cross} to obtain mutual representations from different tasks. Our work differs in both the problem setting as well as the optimization method to obtain a compact and powerful compressed network.

\section{Methods}


Our goal is to obtain a compact student network $\mathcal{F}^S$ from the over-parameterized teacher network $\mathcal{F}^T$.
Given few shot training instances $\{\m x_n, \m y_n\}_{n=1}^{N}$,
we denote their corresponding $l$-th convolutional feature map of the teacher network $\mathcal{F}^T$ as $\m h_{l}^T = \sigma(\m W^T_l * \m h^T_{l-1}) \in R^{N \times c_i \x k \x k}$ , where $\sigma(\cdot)$ is the activation function, $*$ is the convolutional operation, $\m W^T_l \in R^{c_o \x c_i \x k \x k}$ is the 4-D convolutional kernel, and $N$, $c_i, c_o$ and $k$ are the number of training size, input channels, output channels and the kernel size respectively.
Batch normalization layers are omitted as they can be readily fused into convolutional. 
Similar notations hold for $\mathcal{F}^S_l$. In the following, we drop the layer index $l$.

Unlike standard knowledge distillation approaches,
here we adopt layer-wise knowledge distillation which can take layer-wise supervision from the teacher network.
As is shown in Figure~\ref{fig:regu_connect}, with previous layers being fixed, layer-wise distillation aims to find the optimal $\m W^S_*$  that minimizes the Euclidean distance between $\m h^T$ and $\m h^S$, i.e., 
\begin{equation}
	\label{eq:sparse_regression}
	\m W^{S}_* = \arg\min_{\m W^S} \frac{1}{N} \mathcal{L}^r(\m W^S) + \lambda \mathcal{R}(\m W^S),
\end{equation}
where $\mathcal{L}^r(\m W^S)= \| \sigma(\m W^T * \m h^T) - \sigma(\m W^S * \m h^S) \|_F^2$ is the called \textit{estimation error}, and $\mathcal{R}(\m W^S)$ is some regularization tuned by $\lambda$. 
Despite that one can obtain a decent compact network by Equation~\ref{eq:sparse_regression} with abundant training data~\cite{he2017channelprunning,luo2017thinet}, when there are only few shot training instances, the student network $\mathcal{F}^S$ tends to suffer from high estimation errors on the test set as a result of over-fitting.
Moreover, the errors propagate and enlarge layer-wisely~\cite{dong2017learning}, and finally lead to a large performance drop on $\mathcal{F}^S$.





\subsection{Cross Distillation}

\begin{figure*}[t]
	\centering
	\label{fig:4kdloss}
	\subfigure[Layer-wise distillation] { 
		\label{fig:regu_connect}     
		\includegraphics[width=0.20\textwidth]{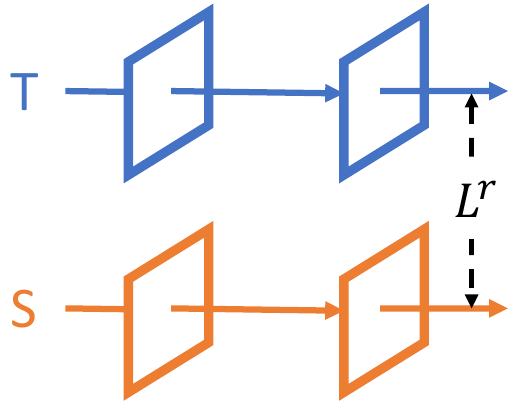}     
	} 
	\subfigure[Correction] { 
		\label{fig:corr_connect}     
		\includegraphics[width=0.24\textwidth]{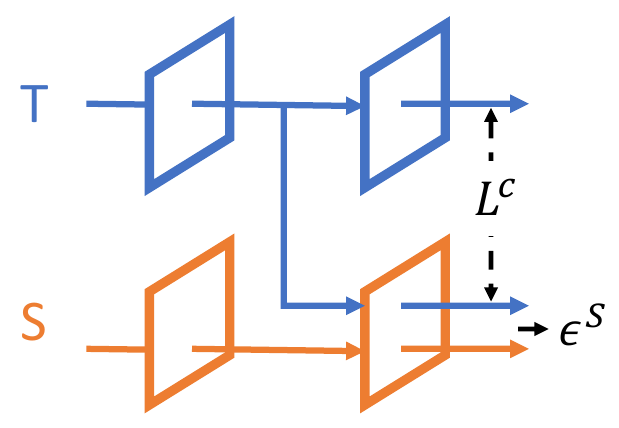}     
	}
	\subfigure[Imitation]{
		\label{fig:imi_connect}     
		\includegraphics[width=0.24\textwidth]{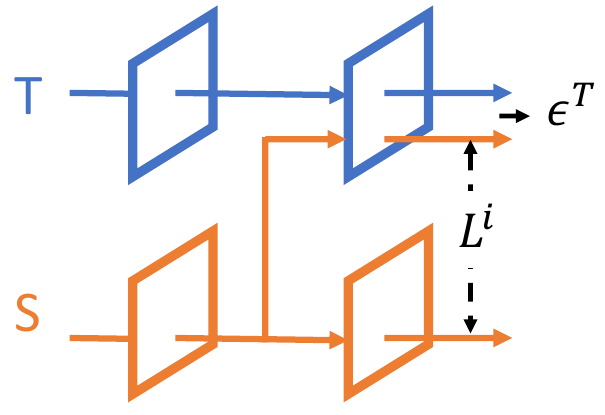}  
	}
	\subfigure[Soft cross distillation]{
		\label{fig:soft_connect}
		\includegraphics[width=0.24\textwidth]{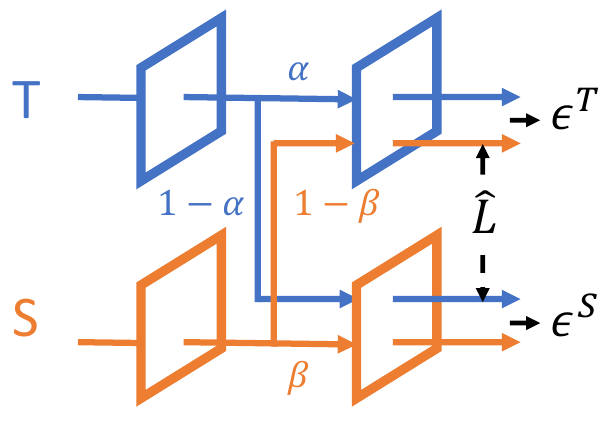}  
	}
	\caption{The four categories of layerwise distillation. a) is the traditional pattern; b) guides the teacher to student in order to compensate estimation errors on the student; c) guides the student to the teacher to make it aware of the errors on the student;
		d) offers a soft connection to balance b) and c) .}
\end{figure*}

To address the above issue, we propose cross distillation, a novel layer-wise distillation method targeting at few shot network compression.
Since the estimation errors are accumulated on the student network $\mathcal{F}^S$ and $\m h^T$ are taken as the target during layer-wise distillation, we direct $\m h^T$ to $\mathcal{F}^S$ in substitution of $\m h^S$ to reduce the historically accumulated errors, as is shown in Figure~\ref{fig:corr_connect}. 
We thus minimize the mean square error of convolutional outputs, which is defined as \textit{correction loss}:
\begin{equation}
	\label{eq:correctoin_loss}
	\mathcal{L}^c(\m W^S) =\|\m W^T * \m h^T - \m W^S * \m h^T\|_F^2.
\end{equation}
In the forward pass of $\mathcal{F}^S$, however, directing $\m h^T$ to $\mathcal{F}^S$ results in inconsistency $\epsilon^S = \|\m W^S * \m h^T - \m W^S * \m h^S \|_F^2$ because $\mathcal{F}^S$ takes $\m h^T$ from $\mathcal{F}^T$ in the training while it is expected to behave along during inference.
Therefore, minimizing the regularized  $\mathcal{L}^c$ could lead to a biasedly-optimized student net.



In order to maintain the consistency during forward pass for $\mathcal{F}^S$ and simultaneously make the teacher aware of the accumulated errors on the student net, we can inverse the strategy by guiding $\m h^S$ to $\mathcal{F}^T$, as is shown in Figure~\ref{fig:imi_connect}. 
We call this process as \textit{imitation}, since the student network tries to mimic the behavior of the teacher network given its current estimations.
Similarly we minimize the mean square error between the corresponding convolutional outputs, defined as the \textit{imitation loss}:
\begin{equation}
	\label{eq:imitation_loss}
	\mathcal{L}^i(\m W^S) = \|\m W^T * \m h^S - \m W^S * \m h^S\|_F^2.
\end{equation}
Despite the teacher network now can provide error-aware supervised signal, such connection brings inconsistency on the teacher network, i.e., 
$\epsilon^T=\|\m W^T * \m h^S - \m W^T * \m h^T\|_F^2$.
As a result of $\epsilon^T$, the errors in $\m h^S$ is be enlarged by $\m W^T$ during layer-wise propagation, leading to deviated supervision for $\mathcal{F}^S$ that deteriorates the distillation.

Consequently, the correction loss $\mathcal{L}^c$ and the imitation loss $\mathcal{L}^i$ compensate each other, and it is necessary to find a proper balance between them.
A natural choice is through convex combination tuned by $\mu$, i.e.
\begin{equation}
	\label{eq:cvx}
	\mathcal{\tilde{L}} = \mu \mathcal{L}^c + (1-\mu) \mathcal{L}^i, \hspace{2ex} \mu \in [0, 1].
\end{equation}
Substituting $\mathcal{L}^r$ in Equation~\ref{eq:sparse_regression} with $\mathcal{\tilde{L}}$ yields the objective function for cross distillation.



\subsubsection{Theoretical Analysis}
The inconsistency gaps $\epsilon^T$ and $\epsilon^S$ of cross distillation make it still unclear how the proposed method manipulates the propagation of estimation errors, and why minimizing the regularized $\mathcal{\tilde{L}}$ is on the right direction to improve the student net $\mathcal{F}^S$.
To theoretically justify cross distillation, we follow ~\cite{friedlander2007exact} to substitute $\mathcal{L}^r$ with $\mathcal{\tilde{L}}$, and equivalently reformulate the unconstrained problem in Equation~\ref{eq:sparse_regression} to the constrained optimization problem as 
\begin{equation}
	\label{eq:constrained_op}
	\min_{\m W^S \in \mathcal{C}} \mathcal{\tilde{L}}, \hspace{2ex} \mathcal{C}=\{ \m W^S \hspace{1ex} | \hspace{1ex} \mathcal{R}(\m W^S) \leq \epsilon(\lambda)\},
\end{equation}
where $\mathcal{C}$ is a compact set determined by the regularization $\mathcal{R}(\m W^S)$ and $\lambda$. 
With Equation~\ref{eq:constrained_op}, we can now bound the gap of cross entropy between $\mathcal{F}^T$ and $\mathcal{F}^S$ for classification\footnote{For regression problems, similar theorem can be established as well.} with the following theorem.

\begin{theorem}
	\label{theorem}
	Suppose both $\mathcal{F}^T$ and $\mathcal{F}^S$ are $L$-layer convolutional neural networks followed by the un-pruned softmax fully-connected layer. If the activation functions  $\sigma(\cdot)$ are Lipchitz-continuous such as $\mathrm{ReLU()}$, the gap of softmax cross entropy  $\mathcal{L}^{ce}$ between the network logits $\m o^T = \mathcal{F}^T(\m x)$ and $\m o^S = \mathcal{F}^S(\m x)$ can be bounded by
	\begin{equation}
		\label{eq:ce_upper_bound}
		| \mathcal{L}^{ce}(\m o^T; \m y) - \mathcal{L}^{ce}(\m o^S; \m y)| \leq C \tilde{\mathcal{L}}_L + \sum_{l=1}^{L-1}\prod_{k=l}^{L} C^{'}_k(\mu) \tilde{\mathcal{L}}_l,
	\end{equation}
	where $C$ and $C^{'}(\mu)$ are constants and $C^{'}(\mu)$ is linear in $\mu$.
\end{theorem}
Theorem~\ref{theorem} shows that 
1) the gap of cross entropy between the student network $\mathcal{F}^S$ and teacher network $\mathcal{F}^T$ is upper bounded by $\tilde{\mathcal{L}}$, and therefore layer-wise minimization of the constrained optimization problem in Equation~\ref{eq:constrained_op} could decrease the gap of cross entropy and finally improve $\mathcal{F}^S$. 
2) The tightness of the upper bound is controlled by the trade-off hyper-parameter $\mu$, which is a $L$-th order polynomial. A proper choice of $\mu$ may lead to a tighter bound that could better decrease the cross entropy gap.
We leave the proof of Theorem~\ref{theorem} in the Appendix.




\paragraph{Soft Cross Distillation}
Although minimizing $\mathcal{\tilde{L}}$ is theoretically supported, the computation of $\mathcal{\tilde{L}}$ involves two loss terms with four convolutions to compute per batch of data, which doubles the training time.
Here we propose another variant to balance $\mathcal{L}^c$ and $\mathcal{L}^i$ by empirically soften the hard connection of $\m h^S$ and $\m h^T$, as is shown in Figure~\ref{fig:soft_connect}. 
We define feature maps $\hat{\m h}^T$ and $\hat{\m h}^S$ after cross connection as the convex combination of $\m h^T$ and $\m h^S$, i.e., 

\begin{equation}
	\label{eq:cvx_combine}
	\left[\begin{matrix}
		\hat{\m h}^T \\
		\hat{\m h}^S
	\end{matrix}
	\right]
	=
	\left [\begin{matrix}
		\alpha & 1-\alpha \\
		1 - \beta & \beta
	\end{matrix}
	\right ]
	\left[\begin{matrix}
		\m h^T \\
		\m h^S
	\end{matrix}
	\right],
\end{equation}
where $\alpha, \beta\in[0, 1]$ are the hyper-parameters that adjust how many percentages are used for cross connection, and therefore the magnitude of inconsistencies $\epsilon^T$ and $\epsilon^S$ can be well controlled. The convex combination ensures the norm of input to be nearly identical after cross connection~(assuming $\|\m h^T\| \approx \|\m h^S \|$), and therefore parameter magnitude stays unchanged. We define the loss of soft cross distillation as
\begin{equation}
	\label{eq:cross_loss}
	\mathcal{\hat{L}}(\m W^S) = \|\sigma(\m W^T * \hat{\m h}^T) -  \sigma(\m W^S *\hat{\m h}^S)\|_F^2,
\end{equation}
which can substitute the estimation error $\mathcal{L}^r$ in Equation~\ref{eq:sparse_regression} as an alternative way for cross distillation.

\subsection{Combined with Network Pruning}

Cross distillation can be readily combined with a set of popular network compression techniques such as pruning or quantization, by taking different regularization $\mathcal{R}(\m W^S)$ in Equation~\ref{eq:sparse_regression}. 
Here we take pruning as an illustration example. 
For non-structured pruning, we choose $\mathcal{R} = \|\m W^S\|_1 = \sum_{i,j,h,w} |W_{ijhw}^S|$; and for structured pruning such as channel pruning, we choose $\mathcal{R}(\m W^S) = \|\m W^S\|_{2,1} = \sum_{i} \|\m W_{i}^S\|_2$, where $\m W_{i}^S \in R^{c_i \x k \x k}$.

To solve Equation~\ref{eq:cross_loss} regularized by the above penalties, we can adopt the proximal gradient method~\cite{parikh2014proximal}, i.e.,  iteratively update $\mathcal{\m W}^S$ by:
\begin{equation}
	\label{eq:prox_update}
	\m W^S_{t+1} = \prox_{\lambda\mathcal{R}}(\m W^S_{t} - \eta \nabla \mathcal{\tilde{L}}(\m W^S_t)),
\end{equation} 
where $\prox_{\lambda \mathcal{R}}(u) = \arg\min_x \frac{1}{2}\|x - u \|_F^2 + \mathcal{R}(x)$ is the proximal operator for $\mathcal{R}$. When $\mathcal{R}$ is chosen as $\|\cdot\|_1$, the proximal mapping can be expressed as the soft-threshold determined by $\lambda$, i.e., 
\begin{equation}
	\small
	\label{eq:soft_thresh}
	\prox_{\lambda \|\cdot\|_1}(W^S_{ijhw}) = \left \{
	\begin{array}{ccl}
		W^S_{ijhw} - \lambda & & W^S_{ijhw} > \lambda \\
		0 & &  |W^S_{ijhw}| \leq \lambda \\
		W^S_{ijhw}  + \lambda & & W^S_{ijhw} < -\lambda 
	\end{array}.
	\right.
\end{equation}
For structured pruning, since $\mathcal{R} = \| \m W^S \|_{2,1}$ is separable w.r.t. $\m W_i^S$, the proximal mapping for $\prox_{\lambda \|\cdot\|_{2,1}} (\m W_i^S)$ can be computed as
\begin{equation}
	\label{eq:group_lasso}
	\prox_{\lambda \|\cdot \|_2} (\m W_i^S) = \max(1 - \frac{\lambda}{\|\m W_i^S\|_2}, 0)\cdot \m W_i^S,
\end{equation}
and the solution to Equation~\ref{eq:prox_update} can be obtained group-wisely from Equation~\ref{eq:group_lasso}. 

As suggested by previous works~\cite{zhu2017prune,he2017channelprunning}, 
we linearly increase $\lambda$ to smoothly prune the student network, which empirically gives better results. Given the 
maximum number of training steps $T$ and the target sparsity ratio $r$ assigned by users, we update $\lambda$ by $\lambda_t = r + (1 - r) * t / T$.
An overall workflow of our proposed method is given in Algorithm~\ref{alg:workflow}.

Finally, we remark that our method works for network quantization as well. By taking $\mathcal{R}(\m W^S)$ as the penalty to quantization points, our method can be combined with Straight Through Estimator (STE)~\cite{bengio2013ste} or ProxQuant~\cite{bai2018proxquant}. See Appendix for details.

\begin{algorithm}[h]
	\caption{Cross distillation}
	\label{alg:workflow}
	\begin{algorithmic}[1]
		\footnotesize
		\REQUIRE ~~\\
		The pre-trained teacher model $\mathcal{F}^T$;\\
		Training samples $\{\m x_n, \m y_n\}_{n=1}^N$; \\
		Target sparsity ratio $r$; 
		\ENSURE  ~~\\
		The compact student model $\mathcal{F}^S$
		\FOR{$l=1,...L$}
		\FOR{$t=1,...T$}
		\STATE Forward pass $\{\m x_n\}_{n=1}^N$ to obtain $\m h_{l-1}^T$ and $\m h_{l-1}^S$;
		\STATE Calculate the loss in Equation~\ref{eq:cvx} or~\ref{eq:cross_loss};
		\STATE Update $\m W^S_t$ with SGD/Adam optimizer;
		\STATE Obtain $\m W^S_{t+1}$ with $\prox_{\lambda \mathcal{R}}$ in Equation~\ref{eq:soft_thresh} or \ref{eq:group_lasso};
		\STATE Increase the pruning threshold $\lambda_t$ linearly;
		\ENDFOR
		\ENDFOR
	\end{algorithmic}
\end{algorithm}

\section{Experiments}
\label{sec:experiment}

We conduct a series of experiments to verify the effectiveness of cross distillation. We take structured and unstructured pruning for demonstration, both of which are popular approaches to reduce computational FLOPs and sizes of neural networks. 
To better understand the proposed method, we also provide further analysis on how cross distillation help reduce the estimation error against varying size of the training set.
Due to limited space, we only present main results, while additional experiments and detailed implementations can be found in the Appendix.
Our implementation in PyTorch is available at \url{https://github.com/haolibai/Cross-Distillation.git}.

\begin{table*}[t]
	\caption{The top-1 accuracy (\%) of structured pruning with VGG-16 on CIFAR-10 with different training sizes. We use VGG-50\% as the pruning scheme, and the original accuracy of the original model is 93.51\%.}
	\label{tab:vgg_cifar_K_shots}
	\footnotesize
	\centering
	\begin{tabular}{ccccccc}  
		\toprule
		Methods  &  1  &  2 & 3 & 5 & 10 & 50 \\
		\midrule
		L1-norm & $14.36_{\pm 0.00}$ & $14.36_{\pm 0.00}$ & $14.36_{\pm 0.00}$ & $14.36_{\pm 0.00}$ & $14.36_{\pm 0.00}$ & $14.36_{\pm 0.00}$\\
		BP & $49.24_{\pm 1.76}$ & $49.32_{\pm 1.88}$ & $51.39_{\pm 1.53}$ & $55.73_{\pm 1.19}$ & $57.48_{\pm 0.91}$ & $64.69_{\pm 0.43}$\\
		FSKD & $47.91_{\pm 1.82}$ & $55.44_{\pm 1.71}$ & $61.76_{\pm 1.39}$ & $65.69_{\pm 1.08}$& $72.20_{\pm 0.74}$ & $75.46_{\pm 0.49}$ \\
		FitNet & $48.51_{\pm 2.51}$ & $71.51_{\pm 2.03}$ & $76.22_{\pm 1.95}$ & $81.10_{\pm 1.13}$ & $85.40_{\pm 1.02}$ & $88.46_{\pm 0.76}$\\
		ThiNet & $58.06_{\pm 1.71}$ & $72.07_{\pm 1.68}$ & $75.37_{\pm 1.59}$ & $78.03_{\pm 1.24}$ & $81.15_{\pm 0.85}$ & $86.12_{\pm 0.45}$\\
		CP  & $66.03_{\pm 1.56}$ & $75.23_{\pm 1.49}$ & $77.98_{\pm 1.47}$ & $81.53_{\pm 1.29}$ & $83.59_{\pm 0.78}$ & $87.27_{\pm 0.27}$\\
		\midrule
		Ours-NC& $65.57_{\pm 1.61}$ & $75.44_{\pm 1.69}$ & $78.40_{\pm 1.53}$ & $81.20_{\pm 1.13}$ & $84.07_{\pm 0.83}$ & $87.67_{\pm 0.29}$\\
		Ours & $\mathbf{69.25_{\pm 1.39}}$ & $\mathbf{80.65_{\pm 1.47}}$ & $\mathbf{82.08_{\pm 1.41}}$ & $\mathbf{84.91_{\pm 0.98}}$ & $\mathbf{86.61_{\pm 0.71}}$ & $87.64_{\pm 0.24}$ \\
		Ours-S & $68.53_{\pm 1.59}$ & $76.83_{\pm 1.43}$ & $80.16_{\pm 1.32}$ & $84.28_{\pm 1.19}$ & $86.30_{\pm 0.79}$ & $\mathbf{88.65_{\pm 0.33}}$  \\
		\bottomrule
	\end{tabular}
\end{table*}

\begin{table*}[t]
	\caption{The top-5 accuracy (\%) of structured pruning with ResNet-34 on ILSVRC-12 with different training sizes. The first three columns use 50, 100 and 500 randomly sampled training instances, while the last three columns use $K=1, 2, 3$ samples per class. We use Res-50\% as the pruning scheme, and the top-1 and top-5 accuracies of the original model are 73.32\% and 91.40\%.}
	\label{tab:res_ilsvrc_K_shots}
	\footnotesize
	\centering
	\begin{tabular}{ccccccc}  
		\toprule
		Methods  &  50  & 100 & 500 & 1 & 2 & 3 \\
		\midrule
		L1-norm & $72.94_{\pm 0.00}$ & $72.94_{\pm 0.00}$ & $72.94_{\pm 0.00}$ & $72.94_{\pm 0.00}$ & $72.94_{\pm 0.00}$ & $72.94_{\pm 0.00}$\\
		BP & $83.18_{\pm 1.86}$ & $84.32_{\pm 1.29}$ & $85.34_{\pm 0.89}$ & $85.76_{\pm 0.73}$ & $86.05_{\pm 0.51}$ & $86.29_{\pm 0.56}$\\
		FSKD & $82.53_{\pm 1.52}$ & $84.58_{\pm 1.13}$ & $86.67_{\pm 0.78}$ & $87.08_{\pm 0.76}$& $87.23_{\pm 0.52}$ & $87.20_{\pm 0.43}$ \\
		FitNet & $86.86_{\pm 1.81}$ & $87.12_{\pm 1.63}$ & $87.73_{\pm 0.96}$ & $87.66_{\pm 0.84}$ & $88.61_{\pm 0.76}$ & $\mathbf{89.32_{\pm 0.78}}$\\
		ThiNet & $85.67_{\pm 1.57}$ & $85.54_{\pm 1.39}$ & $86.97_{\pm 0.89}$ & $87.42_{\pm 0.76}$ & $87.52_{\pm 0.68}$ & $87.53_{\pm 0.50}$\\
		CP  & $86.34_{\pm 1.24}$ & $86.38_{\pm 1.37}$ & $87.41_{\pm 0.80}$ & $88.03_{\pm 0.66}$ & $87.98_{\pm 0.49}$ & $88.21_{\pm 0.37}$\\
		\midrule
		Ours-NC& $86.51_{\pm 1.71}$ & $86.61_{\pm 1.20}$ & $87.92_{\pm 0.75}$ & $87.98_{\pm 0.60}$ & $88.63_{\pm 0.49}$ & $88.82_{\pm 0.38}$\\
		Ours & $86.95_{\pm 1.59}$ & $87.60_{\pm 1.13}$ & $88.34_{\pm 0.69}$ & $88.17_{\pm 0.73}$ & $88.57_{\pm 0.40}$ & $88.59_{\pm 0.41}$ \\
		Ours-S & $\mathbf{87.42_{\pm 1.69}}$ & $\mathbf{87.73_{\pm 1.17}}$ & $\mathbf{88.60_{\pm 0.82}}$ & $\mathbf{88.40_{\pm 0.61}}$ & $\mathbf{88.84_{\pm 0.48}}$ & $88.87_{\pm 0.35}$  \\
		\bottomrule
	\end{tabular}
\end{table*}

\subsection{Setup}
Throughout the experiment, we use VGG~\cite{simonyan2014very} and ResNet~\cite{he2016deep} as base networks, and evaluations are performed on CIFAR-10 and ImageNet-ILSVRC12.
As we consider the setting of few shot image classification, we randomly select $K$-shot instances per class from the training set. All experiments are averaged over five runs with different random seeds, and results of means and standard deviations are reported~\footnote{Note that for each run, we fix the random seed and remove all the randomness such as data augmentation and data shuffling.}

\paragraph{Baselines}
For structured pruning, we compare our proposed methods against a number of baselines: 1) L1-norm pruning~\cite{li2016pruning}, a data-free approach; 2) Back-propagation~(BP) based fine-tuning on L1-norm pruned models; 3) FitNet~\cite{romero2014fitnets} and 4) FSKD~\cite{li2019knowledge}, both of which are knowledge distillation methods; 5) ThiNet~\cite{luo2017thinet} and 6) Channel Pruning~(CP)~\cite{he2017channelprunning}, both of which are layer-wise regression based channel pruning methods. 
For unstructured pruning, we modify 1) to element-wise L1-norm based pruning~\cite{zhu2017prune}. Besides, 4) FSKD, 5) ThiNet and 6) CP are removed since they are only applicable in channel pruning. 

For our proposed method, we compare to three variants for ablation study: Ours-NC (no cross distillation) by solving Equation~\ref{eq:sparse_regression}, Ours by solving Equation~\ref{eq:cvx} and Ours-S~(soft cross distillation) by solving Equation~\ref{eq:cross_loss}.
For Ours, we choose $\mu=0.6$ for VGG networks and $\mu=0.9$ for ResNets. For Ours-S, we set $(\alpha, \beta)=(0.9, 0.3)$ on VGG networks and $(0.9, 0.5)$ on ResNets. Sensitivity analysis on these hyper-parameters are presented later.
Details on parameter settings and baseline implementations are in the Appendix.

\paragraph{Pruning schemes}
The structured pruning schemes are similar to those used in ~\cite{li2016pruning,li2019knowledge}. For the VGG-16 network, we denote the three pruning schemes in~\cite{li2019knowledge} in the ascending order of sparsity as VGG-A, VGG-B and VGG-C respectively. We further prune $50\%$ channels layer-wisely and denote the resulting scheme as VGG-$50\%$. 
For ResNet-34, we remove $r\%$ channels in the middle layer of the first three residual blocks with some sensitive layers skipped (e.g., layer 2, 8, 14, 16). The last residual block is kept untouched. The resulting structured pruning schemes are denoted as Res-$r\%$. Besides, we further remove $50\%$ channels for the last block to reduce more FLOPs when $r=70\%$, denoted as Res-70\%+.
The reduction of model sizes and computational FLOPs for structured pruned models are shown in the Appendix. 

In terms of unstructured pruning, we follow a similar pattern in~\cite{zhu2017prune} by removing $r=\{50\%, 70\%, 90\%, 95\%\}$ parameters for both the VGG network and ResNet, and each layer is treated equally.

\subsection{Results}
\paragraph{Structured Pruning}
We evaluate structured pruning with VGG-16 on CIFAR-10 and ResNet-34 on ILSVRC-12. 
Table~\ref{tab:vgg_cifar_K_shots} and~\ref{tab:res_ilsvrc_K_shots} shows the results with different number of training instances when the pruning schemes are fixed.
It can be observed that both Ours and Ours-S generally outperform the rest baselines on both networks, whereas Ours enjoys a larger advantage on VGG-16 while Ours-S is superior on ResNet-34. 
Meanwhile, as the training size decreases, cross distillation brings more advantages comparing to the rest baselines, indicating that the layer-wise regression can benefit more from cross distillation when the student network over-fits more seriously on fewer training samples.

In the next, we fix the training size and change the pruning schemes. We keep $K=5$ on CIFAR-10 and $K=1$ on ILSVRC-12, and the results are listed in Table~\ref{tab:vgg_cifar_schemes} and Table~\ref{tab:resnet_ilsvrc_schemes} respectively. 
Again on both datasets our proposed cross distillation performs consistently better comparing to the rest approaches. 
Besides, the gain from cross distillation becomes larger as the sparsity of the student network increases~(e.g., VGG-C and ResNet-70\%+). We suspect that networks with sparser structures tend to suffer more from higher estimation errors, which poses more necessity for cross distillation to reduce the errors.

\begin{table*}
	\centering
	\footnotesize
	\setlength{\tabcolsep}{2 pt}{
	\begin{minipage}{0.49\textwidth}
	\caption{The top-1 accuracy (\%) of different structured pruning schemes with VGG-16 on CIFAR-10. 10 samples per class are used.}
	\label{tab:vgg_cifar_schemes}
	\begin{tabular}{lllllll}  
		\toprule
		Methods  &  VGG-50\%  &  VGG-A & VGG-B & VGG-C \\
		\midrule
		L1-norm &  $14.36_{\pm 0.00}$ & $88.32_{\pm 0.00}$ & $32.87_{\pm 0.00}$ & $10.00_{\pm 0.00}$  \\
		BP & $55.73_{\pm 1.19}$ & $93.10_{\pm 0.09}$ & $87.17_{\pm 0.49}$ &  $62.45_{\pm 1.25}$ \\
		FSKD & $65.69_{\pm 1.08}$ & $93.52_{\pm 0.23}$ & $90.69_{\pm 0.12}$ & $81.79_{\pm 1.01}$ \\
		FitNet & $85.40_{\pm 1.02}$ & $93.50_{\pm 0.06}$ & $92.42_{\pm 0.32}$ & $84.65_{\pm 1.53}$ \\
		ThiNet & $81.15_{\pm 0.85}$ & $93.61_{\pm 0.05}$ & $92.20_{\pm 0.16}$ & $79.19_{\pm 0.91}$\\
		CP  & $83.59_{\pm 0.78}$ & $\mathbf{93.70_{\pm 0.04}}$ & $92.29_{\pm 0.15}$ & $80.82_{\pm 0.73}$ & \\
		\midrule
		Ours-NC & $84.07_{\pm 0.83}$ & $\mathbf{93.69_{\pm 0.07}}$ & $92.35_{\pm 0.14}$ & $83.90_{\pm 0.78}$ \\
		Ours & $\mathbf{86.61_{\pm 0.71}}$ & $\mathbf{93.65_{\pm 0.08}}$ & $\mathbf{92.60_{\pm 0.11}}$ & $\mathbf{85.81_{\pm 0.80}}$ \\
		Ours-S & $86.30_{\pm 0.79}$ & $\mathbf{93.70_{\pm 0.07}}$ & $\mathbf{92.68_{\pm 0.13}}$  &  $85.10_{\pm 0.75}$ \\
		\bottomrule
	\end{tabular}
	\end{minipage}
	\begin{minipage}{0.49\textwidth}
	\caption{The top-5 accuracy (\%) of different structured pruning schemes with ResNet-34 on ILSVRC-12. 1 sample per class is used.}
	\label{tab:resnet_ilsvrc_schemes}
	\begin{tabular}{lllllll}  
	\toprule
	Methods  &  Res-30\%  &  Res-50\% & Res-70\% & Res-70\%+ \\
	\midrule
	L1-norm &  $84.54_{\pm 0.00}$ & $72.94_{\pm 0.00}$ & $31.84_{\pm 0.00}$ & $15.30_{\pm 0.00}$  \\
	BP & $88.66_{\pm 0.59}$ & $85.76_{\pm 0.73}$ & $80.04_{\pm 0.90}$ &  $63.25_{\pm 1.05}$ \\
	FSKD & $89.56_{\pm 0.52}$ & $87.08_{\pm 0.76}$ & $80.82_{\pm 0.62}$ & $67.04_{\pm 0.56}$ \\
	FitNet & $88.56_{\pm 0.58}$ & $87.66_{\pm 0.84}$ & $\mathbf{82.72_{\pm 0.88}}$ & $68.31_{\pm 0.81}$ \\
	ThiNet & $89.74_{\pm 0.65}$ & $87.42_{\pm 0.76}$ & $79.40_{\pm 0.66}$ & $63.65_{\pm 0.78}$\\
	CP  & $89.65_{\pm 0.78}$ & $88.03_{\pm 0.66}$ & $81.13_{\pm 0.85}$ & $68.18_{\pm 0.79}$ & \\
	\midrule
	Ours-NC& $\mathbf{90.34_{\pm 0.53}}$ & $87.98_{\pm 0.60}$ & $82.11_{\pm 0.71}$ & $69.03_{\pm 0.92}$ \\
	Ours & $90.08_{\pm 0.47}$ & $88.17_{\pm 0.65}$ & $\mathbf{82.71}_{\pm 0.76}$ & $\mathbf{73.53}_{\pm 0.74}$ \\
	Ours-S & $\mathbf{90.32_{\pm 0.58}}$ & $\mathbf{88.40_{\pm 0.61}}$ & $82.65_{\pm 0.68}$  &  $69.47_{\pm 0.79}$ \\
	\bottomrule
	\end{tabular}
	\end{minipage}}
\end{table*}

\begin{table*}[t]
	\caption{The top-5 accuracy (\%) of unstructured pruning with VGG-16 on ILSVRC-12 with different training sizes. The first three columns use 50, 100 and 500 randomly sampled training instances, while the last three columns use $K =1, 2, 3$ samples per class. We use Res-90\% as the pruning scheme, and the top-1 and top-5 accuracies of the original model are 73.72\% and 91.51\%.}
	\label{tab:vgg_ilsvrc_un_K_shot}
	\centering
	\footnotesize
	\begin{tabular}{ccccccc}  
		\toprule
		Methods  &  50 &  100 & 500 & 1 & 2 & 3 \\
		\midrule
		L1-norm &  $0.5_{\pm 0.00}$ & $0.5_{\pm 0.00}$  & $0.5_{\pm 0.00}$ &  $0.5_{\pm 0.00}$ &  $0.5_{\pm 0.00}$ &  $0.5_{\pm 0.00}$\\
		BP & $42.87_{\pm 2.07}$ & $48.78_{\pm 1.43}$ & $65.47_{\pm 1.15}$ & $71.25_{\pm 0.97}$ & $74.85_{\pm 0.71}$ & $76.04_{\pm 0.48}$\\
		FitNet & $52.66_{\pm 2.93}$ & $57.09_{\pm 2.14}$ & $76.59_{\pm 1.45}$ & $80.14_{\pm 1.23}$ & $82.27_{\pm 0.70}$ & $83.14_{\pm 0.51}$\\
		\midrule
		Ours-NC& $78.73_{\pm 1.78}$ & $83.29_{\pm 1.12}$ & $85.04_{\pm 0.93}$ & $85.36_{\pm 0.61}$ & $85.21_{\pm 0.41}$ & $85.49_{\pm 0.46}$\\
		Ours & $\mathbf{83.81_{\pm 1.49}}$ & $86.21_{\pm 1.09}$ & $87.19_{\pm 0.96}$ & $87.61_{\pm 0.82}$ & $87.78_{\pm 0.45}$ & $87.86_{\pm 0.39}$ \\
		Ours-S & $83.67_{\pm 1.52}$ & $\mathbf{86.72_{\pm 1.23}}$ & $\mathbf{87.82_{\pm 1.04}}$ & $\mathbf{88.14_{\pm 0.74}}$ & $\mathbf{88.23_{\pm 0.61}}$ & $\mathbf{88.38_{\pm 0.43}}$  \\
		\bottomrule
	\end{tabular}
\end{table*}

\paragraph{Unstructured Pruning}
For unstructured pruning, here we present results of the VGG-16 network on ILSVRC-12 dataset.
Similar to structured pruning, we first fix the pruning scheme and vary the training size, and the results are given in Table~\ref{tab:vgg_ilsvrc_un_K_shot}. It can be observed that both Ours and Ours-S significantly outperform the rest methods, and the improvement is even larger comparing to structured pruning
One reason could be the irregular sparsity of network parameters cab better compensate the layer-wisely accumulated errors on $\mathcal{F}^S$.

Similarly, we test our methods with different sparsities and hold the training size fixed as $K=1$, and Table~\ref{tab:vgg_ilsvrc_un_schemes} shows the results.
As the sparsity $r$ increases, cross distillation brings more improvement, especially on  VGG-95\% with a nearly 10\% and 14\% increase of accuracy for Ours and Ours-S respectively.

\begin{table}
	\caption{The top-5 accuracy (\%) of unstructured pruning with VGG-16 on ILSVRC-12 with different pruning schemes. 1 sample per class is adopted.}
	\label{tab:vgg_ilsvrc_un_schemes}
	\centering
	\footnotesize
	\setlength{\tabcolsep}{2 pt}{
	\begin{tabular}{ccccccc}  
		\toprule
		Methods  &  VGG-50\%  &  VGG-70\% & VGG-90\% & VGG-95\% \\
		\midrule
		L1-norm & $89.21_{\pm 0.00}$ & $66.91_{\pm 0.00}$ & $0.5_{\pm 0.00}$ & $0.50_{\pm 0.00}$  \\
		BP & $90.61_{\pm 0.20}$ & $88.08_{\pm 0.19}$ & $71.25_{\pm 0.97}$ &  $42.37_{\pm 1.59}$ \\
		FitNet & $88.36_{\pm 0.46}$ & $86.76_{\pm 0.67}$ & $80.14_{\pm 1.23}$ & $59.08_{\pm 1.78}$ \\
		\midrule
		Ours-NC& $91.47_{\pm 0.12}$ & $91.16_{\pm 0.10}$ & $85.21_{\pm 0.41}$ & $66.74_{\pm 1.36}$ \\
		Ours & $91.58_{\pm 0.06}$ & $91.24_{\pm 0.14}$ & $87.61_{\pm 0.49}$ & $76.65_{\pm 1.23}$ \\
		Ours-S & $\mathbf{91.68_{\pm 0.09}}$ & $\mathbf{91.54_{\pm 0.11}}$ & $\mathbf{88.14_{\pm 0.61}}$  &  $\mathbf{80.64_{\pm 1.03}}$ \\
		\bottomrule
	\end{tabular}}
\end{table}

\subsection{Further Analysis}
\paragraph{The Estimation Errors v.s. Inconsistency}
Cross distillation brings the inconsistencies $\epsilon^T$, $\epsilon^S$ that could affect the reduction of estimation errors $\mathcal{L}^r$.
To quantitatively investigate the effects, we compare $\epsilon^T$, $\epsilon^S$ as well as $\mathcal{L}^r$ at different layers of the VGG-16 network on the test set of CIFAR-10.
We take three student networks trained by the correction loss $\mathcal{L}^c$, the imitation loss $\mathcal{L}^i$ as well as soft distillation loss $\mathcal{\hat{L}}$ respectively.
We choose unstructured pruning with VGG-90\% and vary K between $\{1, 10\}$, and the results are shown in Figure~\ref{fig:teacher_gap},~\ref{fig:student_gap} and~\ref{fig:regu_loss} respectively. 
Note that we have normalized the loss values by dividing the nonzero leftmost bar in each sub-figure.

It can be observed that the student net trained by $\mathcal{L}^c$ has a large $\epsilon^S$ with $\epsilon^T=0$, and vice versa for that trained by $\mathcal{L}^i$. On the contrary, the student net trained by $\mathcal{\hat{L}}$ shows both lower $\epsilon^T$ and $\epsilon^S$, and the estimation error $\mathcal{L}^r$ is properly reduced as well. 
The results indicate that by properly controlling the magnitude of inconsistencies $\epsilon^T$ and $\epsilon^S$ with soft connection, cross distillation can indeed reduce estimation errors $\mathcal{L}^r$ and improve the student network.

\begin{figure*} [t]
	\centering 
	\label{fig:loss_story}
	\subfigure[$\epsilon^T$] { 
		\label{fig:teacher_gap}     
		\includegraphics[width=0.65\columnwidth]{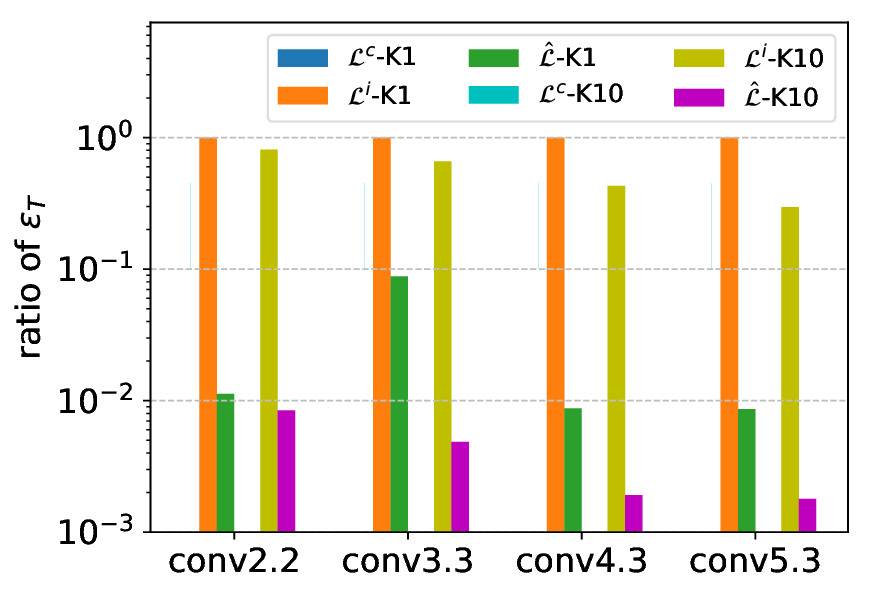}     
	} 
	\subfigure[$\epsilon^S$] { 
		\label{fig:student_gap}     
		\includegraphics[width=0.65\columnwidth]{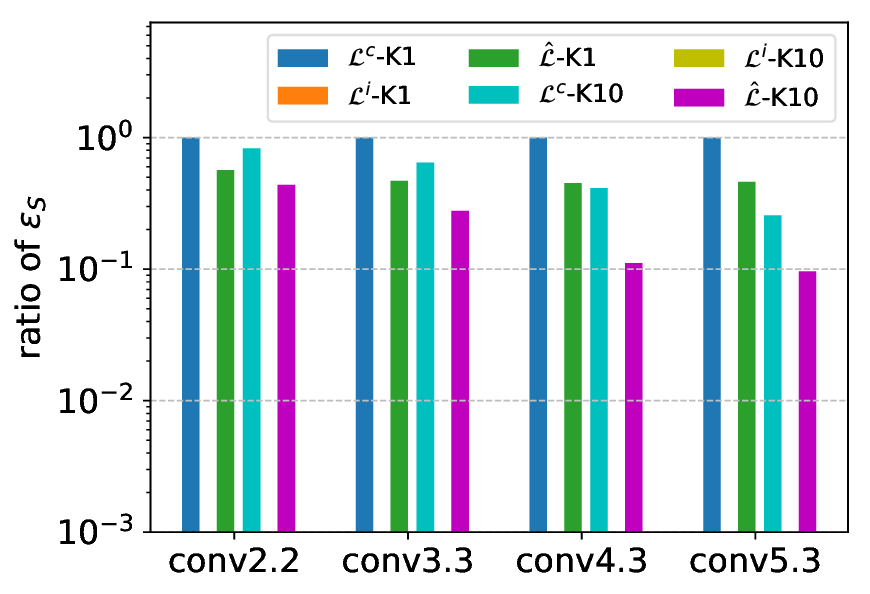}     
	}
	\subfigure[$\mathcal{L}^r$]{
		\label{fig:regu_loss}     
		\includegraphics[width=0.65\columnwidth]{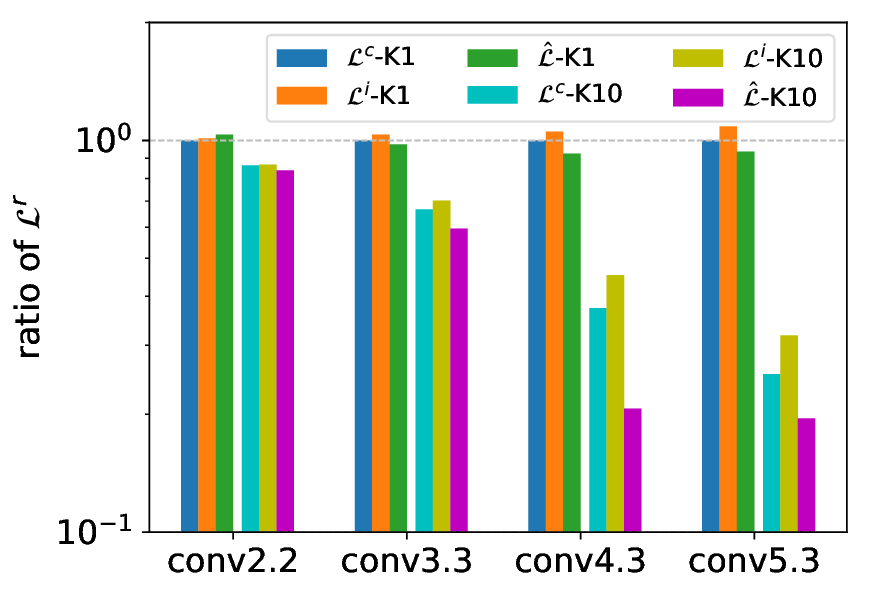}  
	}
	\caption{The comparisons among inconsistencies  $\epsilon^T$, $\epsilon^S$ as well as estimation errors $\mathcal{L}^r$ on the test set of CIFAR-10. The colors denote what kind of loss and values of $K$ are adopted for training. Best viewed in color.}
\end{figure*}

\begin{figure}
	\centering
	\includegraphics[width=0.95\columnwidth]{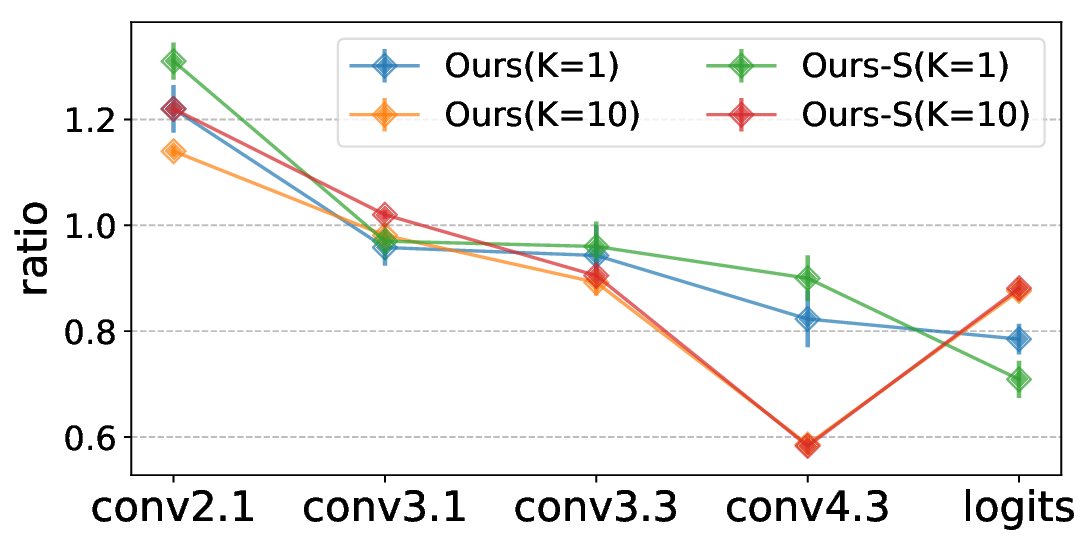}
	\caption{The estimation errors $\mathcal{L}^r$ of Ours and Ours-S, both of which are divided by Ours-NC. Best viewed in color.}
	\label{fig:estimation_inconsistency}
\end{figure}

\paragraph{Generalization Ability}
One potential issue troubles us is the generalization of cross distillation, since the training of Ours and Ours-S is somehow biased comparing to Ours-NC that directly minimizes the estimation error $\mathcal{L}^r$.
Since estimation errors $\mathcal{L}^r$ among feature maps and cross entropy $\mathcal{L}^{ce}$ of logits directly reflect the closeness between $\mathcal{F}^T$ and $\mathcal{F}^S$ during inference, we compare both results among student nets obtained by Ours-NC, Ours and Ours-S respectively.
We again take unstructured pruning with VGG-90\% on the test set of CIFAR-10, and the rest settings are kept unchanged.
For ease of comparison, we similarly divide values of Ours and Ours-S by those obtained by Ours-NC. Ratios smaller than 1 indicate a more generalizable student net.

From Figure~\ref{fig:estimation_inconsistency}, we can find that while the ratios in shallower layers are above 1, they rapidly go down at deeper layers such as conv4.3 as well as the logits, which is consistent with Figure~\ref{fig:estimation_inconsistency} that cross distillation tends to better benefit  deeper layers.
Moreover, although increasing $K$ from 1 to 10 gives lower ratios of $\mathcal{L}^r$ at convolutional layers, the ratios of $\mathcal{L}^{ce}$ increases at the network logits, which lead to less improvement for classification when more training samples are available. The phenomenons are consistent with the results in Table~\ref{tab:vgg_cifar_K_shots}, Table~\ref{tab:res_ilsvrc_K_shots} and Table~\ref{tab:vgg_ilsvrc_un_K_shot}.
In summary, cross distillation can indeed generalize well when $\mathcal{F}^T$ and $\mathcal{F}^S$ are properly mixed in the few shot setting.

\subsubsection{Sensitivity Analysis}
Finally, we present sensitivity analysis for cross distillation. We perform grid search by varying $\mu\in [0,1]$ for Ours and $(\alpha, \beta)\in [0,1]^2$ for Ours-S at an interval of $0.1$. 
We take VGG-16 for structured pruning and ResNet-56 for unstructured pruning on CIFAR-10 with $K=5$, while ILSVRC-12 experiments adopt the same setting of $\mu$ and $(\alpha,\beta)$ found by these experiments. 
From Figure~\ref{fig:cvx_vgg_grid} and Figure ~\ref{fig:cvx_resnet_grid}, Ours consistently outperforms Ours-NC, where the best configurations appear at around $\mu=0.6$ for VGG-16 and $\mu=0.9$ for ResNet-56.
Furthermore, we find that simply using the correction loss $\mu=0.0$ or the imitation loss $\mu=1.0$ also achieve reasonable results\footnote{The accuracies are $83.44\%$ and $83.32\%$ respectively on VGG-16, and $84.93\%$ and $86.63\%$ respectively on ResNet-56.}. 
In terms of Ours-S in Figure~\ref{fig:vgg_grid} and ~\ref{fig:resnet_grid} , we find that on left regions $\{(\alpha, \beta)|\alpha+\beta < 1\}$  $\mathcal{F}^T$ and $\mathcal{F}^S$ permute the input too much and thereon lead to significant drops of performance.
For right regions $\{(\alpha, \beta)|\alpha+\beta > 1\}$, most configurations consistently outperform Ours-NC $(1.0, 1.0)$, and the peaks occur somewhere in the middle of the regions.

\begin{figure}[t]\centering 
	\label{fig:grid_search}
	\subfigure[Ours on VGG-16]{
		\label{fig:cvx_vgg_grid}     
		\includegraphics[width=0.46\columnwidth]{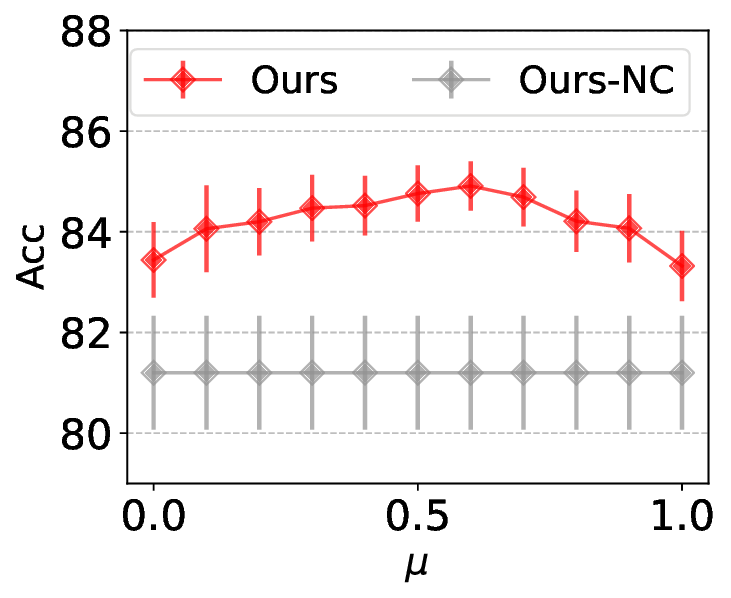}  
	}
	\subfigure[Ours on ResNet-56]{
		\label{fig:cvx_resnet_grid}     
		\includegraphics[width=0.46\columnwidth]{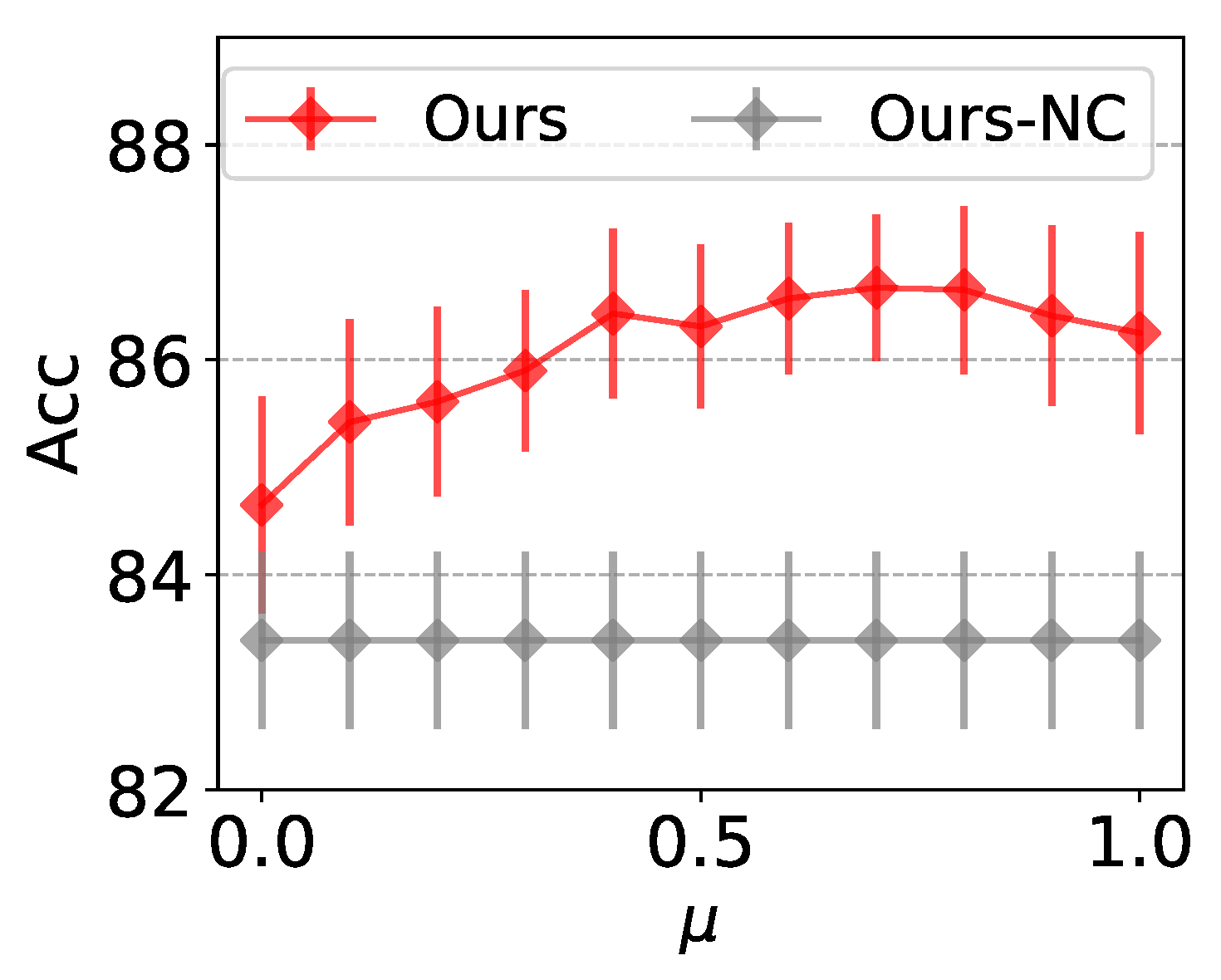}  
	}
	\subfigure[Ours-S on VGG-16]{
		\label{fig:vgg_grid}     
		\includegraphics[width=0.46\columnwidth]{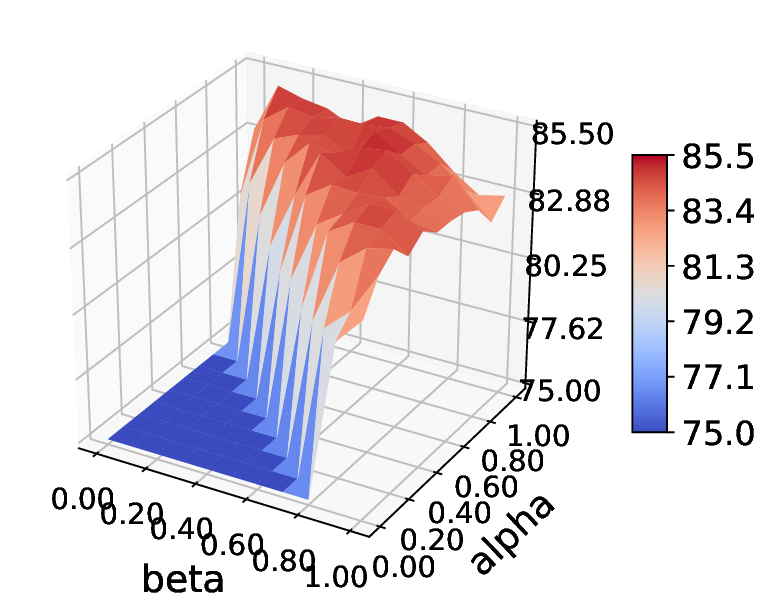}  
	}     
	\subfigure[Ours-S on ResNet-56] { 
		\label{fig:resnet_grid}     
		\includegraphics[width=0.46\columnwidth]{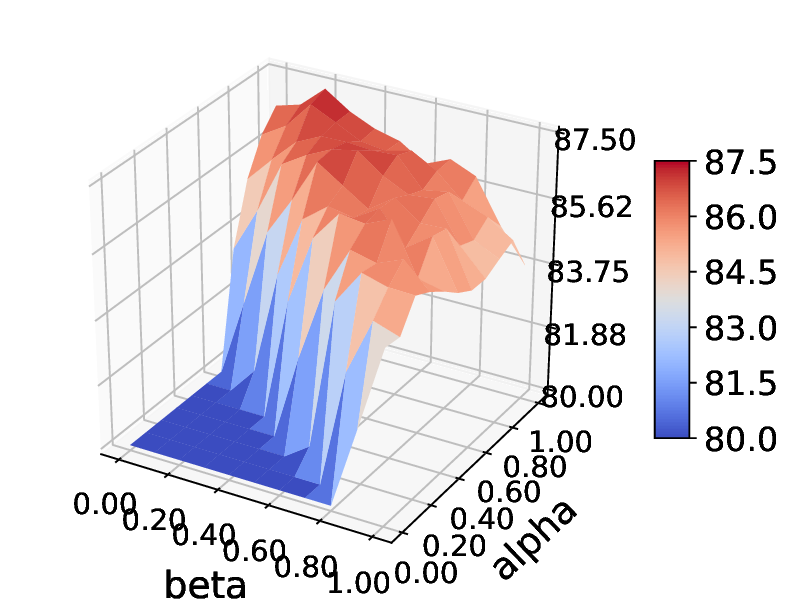}     
	}    
	\caption{ Sensitivity analysis of $\mu\in[0,1]$ for Ours and $(\alpha, \beta)$ on $[0, 1]^2$ for Ours-S.}
\end{figure}

\section{Conclusion}

In this paper, we present cross distillation, a novel knowledge distillation approach for learning compact student network given limited number of training instances. 
By reducing estimation errors between the student network and teacher network, cross distillation can bring a more powerful and generalizable student network.
Extensive experiments on benchmark datasets demonstrate the superiority of our method against various competitive baselines.

\paragraph{Acknowledgement}
This work is supported by the Research Grant Coucil of the Hong Kong Special Administrative Region, China (No.CUHK 14208815 and No.CUHK 14210717 of the General Research Fund). We sincerely thank Xin Dong, Jiajin Li, Jiaxing Wang and Shilin He for helpful discussions, as well as the anonymous reviewers for insightful suggestions.

\bibliographystyle{aaai}
\bibliography{bibfile}

\section{Appendix}
\label{appendix}

\section{Proof to Theorem~\ref{theorem}}
We decompose the proof to Theorem~\ref{theorem} into two parts. We first show the Lipchitz continuity for the softmax cross entropy function in Lemma~\ref{lemma1}, then we show the layer-wise propagation of estimation errors in a recursive way in Lemma~\ref{lemma2}.
Theorem~\ref{theorem} can be easily verified by combining Lemma~\ref{lemma1} and Lemma~\ref{lemma2}.

\begin{lemma}
\label{lemma1}
For the network logits $\m o = \mathcal{F}(\m x) \in R^d$ and labels $\m y \in R^d$, the softmax cross entropy $\mathcal{L}^{ce}(\m o; \m y)=-\sum_{i=1}^d y_i\log \frac{\exp(o_i)}{\sum_{j=1}^d \exp(o_j)}$ is $C$-Lipchitz continuous for some constant $C>0$.
\end{lemma}

\begin{proof}
Note that 
\begin{equation}
\label{eq:cross_entropy}
\mathcal{L}^{ce}(\m o; \m y) = -\sum_{i=1}^d y_i o_i + \log{\sum_{j=1}^d \exp(o_j)},
\end{equation}
where we have used the fact $\sum_{i=1}^d y_i = 1$ since $\m y$ is a one-hot vector. The first term is linear in $\m o$ and therefore satisfies the Lipchitz continuity. We now turn to verify the Lipchitz continuity of the function $\phi (\m o) = \log \sum_{i=1}^d \exp(o_i)$.
According to the intermediate value theorem, for $\forall \m o^S, \m o^T \in R^d$, $\exists t\in [0, 1]$ such that for $\bar{\m o} = t \m o^T + (1-t) \m o^S$, we have
\begin{align}
\label{eq:logsumexp_lipchitz}
& |\phi(\m o^T) - \phi(\m o^S)|  \nonumber \\
& = |\nabla \phi(\bar{\m o})^{\top} (\m o^T - \m o^S)| \nonumber \\
& \leq \|\nabla \phi(\bar{\m o} )\|_1 \|\m o^T - \m o^S\|_{\infty} \nonumber \\
& =\|\m o^T - \m o^S \|_{\infty} \nonumber \\
& \leq C_0 \|\m o^T - \m o^S\|, 
\end{align}
where the third line comes from the Holder's inequality, the fourth line comes from the fact that $\nabla\phi(\bar{\m o} )$ is a softmax function lying on a simplex, and the last line is due to the equivalence among norms.
With Equation~\ref{eq:logsumexp_lipchitz}, one can easily verify that 
\begin{align}
    \label{eq:cross_entropy_lipchitz}
    & |\mathcal{L}^{ce}(\m o^T; \m y) - \mathcal{L}^{ce}(\m o^S; \m y)| \nonumber \\
    & = |\sum_{i=1}^d y_i (o_i^T - o_i^S) + \phi(\m o^T) - \phi(\m o^S)| \nonumber \\
    & \leq \|\m y\| \|\m o^T - \m o^S\| + |\phi(\m o^T) - \phi(\m o^S) | \nonumber \\
    & \leq (C_0 + 1) \|\m o^T - \m o^S \| \nonumber \\
    & = (C_0+1) \|\m W \m h_L^T - \m W \m h_L^S \| \nonumber \\
    & \leq C \|\m h_L^T - \m h_L^S\|
\end{align}
where we have used facts that $\|\m y\|=1$, $\m o^T = \m W \m h_L^T$, $\m o^S = \m W \m h_L^S$ with $\m W$ as shared parameters of the last layer, and $C = (C_0 + 1)\cdot \|\m W\|$. 

\end{proof}

\begin{lemma}
\label{lemma2}
Suppose both $\mathcal{F}^T$ and $\mathcal{F}^S$ are activated by the Lipchitz-continuous function $\sigma(\cdot) = \mathrm{ReLU}(\cdot)$, the estimation error $\mathcal{L}_L^r$ at layer $L$ can be bounded by the layerwise objective function $\mathcal{\tilde{L}}_l$ as follows:
\begin{equation}
    \label{eq:error_propagate}
    \mathcal{L}^r_L \leq \sum_{l=1}^{L-1}\prod_{k=l}^{L}C_k(\mu)\mathcal{\tilde{L}}_l +  \mathcal{\tilde{L}}_L,
\end{equation}
where $C_k(\mu)$ is some constant linear in $\mu$ in the $k$-th layer.
\end{lemma}

\begin{proof}
Recall that $\m h^T_L = \sigma(\m W^T_L * \m h^T_{L-1})$. To facilitate the following analysis, we apply the im2col operation to equivalently transform the convolution to matrix multiplication, i.e. $\bar{\m h}^T_L = \sigma(\bar{\m W}^T_L \bar{\m h}^T_{L-1})$, where $\bar{\m h}^T_L \in R^{c_o \times (N  c_i k k)}$ and $\bar{\m W}_L^T \in R^{c_o \times (c_i k k)}$ are matrices. Then
\begin{align}
    \label{eq:triangle_1}
    & \mathcal{L}^r_L = \|\m h_L^T - \m h_L^S\| = \|\bar{\m h}_L^T - \bar{\m h}_L^S\| \nonumber \\
    & = \|\sigma(\bar{\m W}_{L}^T \bar{\m h}_{L-1}^T) - \sigma(\bar{\m W}_{L}^S\bar{\m h}_{L-1}^S)\| \nonumber \\
    & \leq \|\bar{\m W}_{L}^T\bar{\m h}_{L-1}^T - \bar{\m W}_{L}^S\bar{\m h}_{L-1}^S \| \nonumber \\
    & = \|\bar{\m W}_{L}^T\bar{\m h}_{L-1}^T - \bar{\m W}^S \bar{\m h}_{L-1}^T +  \bar{\m W}^S \bar{\m h}_{L-1}^T - \bar{\m W}_{L}^S\bar{\m h}_{L-1}^S \| \nonumber \\
    & \leq \mathcal{L}^c_{L-1} + \epsilon^S_{L-1} \nonumber\\
    & \leq \mathcal{L}^c_{L-1} + \|\bar{\m W}_{L}^S \|\cdot \mathcal{L}^r_{L-1},
\end{align}
where the first inequality comes from the Lipchitz continuity of the $\mathrm{ReLU}(\cdot)$ function, and the rest can be readily obtained by applying the triangle inequality.
Similarly, we have 
\begin{equation}
    \label{eq:triangle_2}
    \mathcal{L}^r_L \leq \mathcal{L}^i_{L-1} + \|\bar{\m W}_{L}^T \|\cdot \mathcal{L}^r_{L-1}
\end{equation}
By taking the convex combination of Equation~\ref{eq:triangle_1} and Equation~\ref{eq:triangle_2} for some $\mu\in [0, 1]$, we have
\begin{align}
    \label{eq:recursive_bound}
    & \mathcal{L}^r_{L} \leq \mu (\mathcal{L}^i_{L-1} + \|\bar{\m W}_{L}^S \|\cdot \mathcal{L}^r_{L-1}) \nonumber \\ 
    & + (1-\mu)(\mathcal{L}^i_{L-1} + \|\bar{\m W}_{L}^T \|\cdot \mathcal{L}^r_{L-1}) \nonumber \\
    & \leq \mathcal{\tilde{L}}_{L} + C_L(\mu)\mathcal{L}^r_{L-1}  \nonumber \\
    & \leq \mathcal{\tilde{L}}_{L} + \sum_{l=1}^{L-1}\prod_{k=l}^L C_k(\mu)\mathcal{\tilde{L}}_l,
\end{align}
where we define $C_k(\mu) = \mu \|\bar{\m W}^S_k\| + (1-\mu) \|\bar{\m W}^T_k\|$, and the last line is obtained recursively with $\mathcal{L}^r_0=\|\m x - \m x\| = 0$ at the input of $\mathcal{F}^T$ and $\mathcal{F}^S$.
\end{proof}

Finally, by combining Equation~\ref{eq:cross_entropy_lipchitz} with Equation~\ref{eq:recursive_bound} together and define 
$C^{'}(k)=C\cdot C(k)$, Equation~\ref{eq:ce_upper_bound} in Theorem~\ref{theorem} can be readily verified.


\section{Cross Distillation with Quantization}
To arm cross distillation with network quantization, one can simply take $\mathcal{R}(\m W^S)=\Pi_Q(g(\m W^S))$ in Equation~\ref{eq:sparse_regression}, where $\mathcal{Q}=\{0, \frac{\pm 1}{2^{B-1}-1}, \frac{\pm 2}{2^{B-1}-1}, ..., \pm 1\}$ is the
collection of $2^B$ quantization points, $\Pi_Q$ denotes projection onto $Q$, and $g(\cdot)$ is some transformation function to normalize the input. Optimizing loss functions of cross distillation in Equation~\ref{eq:cvx} or Equation~\ref{eq:cross_loss} is similar to STE training~\cite{bengio2013ste}, where the proximal step performs lazy projection that corresponds to the quantization step in the forward pass in STE.

Similarly, if one take $\mathcal{R}(\m W^S)=\|g(\m W^S) - Q\|_F^2$, the entire procedure reduces to exactly ProxQuant~\cite{bai2018proxquant} which alternates between the proximal step and gradient descent step. 

\section{Additional Experiments}

\begin{table*}[t]
	\footnotesize
	\caption{The top-1 accuracy (\%) on structured pruning using ResNet-56 on CIFAR-10 with different training sizes. We choose Res-50\% as the pruning scheme, and the accuracy of the original model is 93.32\%.}
	\label{tab:resnet_cifar_K_shots}
	\centering
	\begin{tabular}{ccccccc}  
		\toprule
		Methods  &  1  &  2 & 3 & 5 & 10 & 50 \\
		\midrule
		L1-norm & $80.43_{\pm 0.00}$ & $80.43_{\pm 0.00}$ & $80.43_{\pm 0.00}$ & $14.36_{\pm 0.00}$ & $80.43_{\pm 0.00}$ & $80.43_{\pm 0.00}$\\
		BP & $84.17_{\pm 1.55}$ & $86.61_{\pm 1.69}$ & $86.86_{\pm 1.30}$ & $87.41_{\pm 0.98}$ & $87.79_{\pm 0.81}$ & $90.12_{\pm 0.70}$\\
		FSKD & $84.26_{\pm 1.42}$ & $85.79_{\pm 1.31}$ & $85.99_{\pm 1.29}$ & $87.53_{\pm 1.06}$& $88.15_{\pm 0.71}$ & $88.70_{\pm 0.55}$ \\
		FitNet & $86.85_{\pm 1.91}$ & $87.95_{\pm 2.13}$ & $88.94_{\pm 1.85}$ & $89.43_{\pm 1.60}$ & $\mathbf{91.03_{\pm 1.14}}$ & $\mathbf{91.89_{\pm 0.87}}$\\
		ThiNet & $88.40_{\pm 1.26}$ & $88.76_{\pm 1.18}$ & $88.95_{\pm 1.19}$ & $89.54_{\pm 0.84}$ & $90.36_{\pm 0.76}$ & $90.89_{\pm 0.49}$\\
		CP  & $88.53_{\pm 1.37}$ & $88.69_{\pm 1.09}$ & $88.79_{\pm 0.94}$ & $89.39_{\pm 0.80}$ & $89.91_{\pm 0.69}$ & $90.45_{\pm 0.43}$\\
		\midrule
		Ours-NC& $88.05_{\pm 1.61}$ & $88.63_{\pm 1.69}$ & $89.01_{\pm 1.53}$ & $89.51_{\pm 1.13}$ & $90.26_{\pm 0.83}$ & $90.98_{\pm 0.29}$\\
		Ours & $88.42_{\pm 1.63}$ & $89.12_{\pm 1.57}$ & $\mathbf{89.75_{\pm 1.50}}$ &  $89.93_{\pm 1.03}$ & $90.42_{\pm 0.86}$ & $90.85_{\pm 0.24}$ \\
		Ours-S & $\mathbf{89.00_{\pm 1.59}}$ & $\mathbf{89.45_{\pm 1.43}}$ & $89.56_{\pm 1.32}$ & $\mathbf{90.14_{\pm 1.19}}$ & $90.82_{\pm 0.79}$ & $91.24_{\pm 0.33}$  \\
		\bottomrule
	\end{tabular}
\end{table*}

\begin{table*}[t]
	\centering
	\footnotesize
	\setlength{\tabcolsep}{2 pt}{
	\begin{minipage}{0.47\textwidth}
	\caption{The accuracy (\%) of unstructured pruning with VGG-16 on CIFAR-10 with different pruning schemes. $K=5$ samples per class are adopted.}
	\label{tab:vgg_cifar_un_schemes}
	\begin{tabular}{ccccccc}  
		\toprule
		Methods  &  VGG-50\%  &  VGG-70\% & VGG-90\% & VGG-95\% \\
		\midrule
		L1-norm &  $92.47_{\pm 0.00}$ & $85.21_{\pm 0.00}$ & $15.06_{\pm 0.00}$ & $10.00_{\pm 0.00}$  \\
		BP & $93.49_{\pm 0.09}$ & $92.39_{\pm 0.17}$ & $71.38_{\pm 1.16}$ &  $42.16_{\pm 1.98}$ \\
		FitNet & $93.27_{\pm 0.15}$ & $92.51_{\pm 0.17}$ & $83.50_{\pm 1.73}$ & $63.48_{\pm 1.45}$ \\
		\midrule
		Ours-NC& $93.46_{\pm 0.06}$ & $93.17_{\pm 0.11}$ & $89.03_{\pm 0.82}$ & $75.20_{\pm 1.02}$ \\
		Ours & $93.51_{\pm 0.04}$ & $\mathbf{93.30_{\pm 0.09}}$ & $\mathbf{90.62_{\pm 0.93}}$ & $\mathbf{82.81_{\pm 1.26}}$ \\
		Ours-S & $\mathbf{93.58_{\pm 0.06}}$ & $\mathbf{93.27_{\pm 0.10}}$ & $90.36_{\pm 0.89}$  &  $81.97_{\pm 1.13}$ \\
		\bottomrule
	\end{tabular}
	\end{minipage}\hspace{4ex}
	\begin{minipage}{0.47\textwidth}
	\caption{The accuracy (\%) of unstructured pruning with ResNet-56 on CIFAR-10 with different pruning schemes. $K=5$ samples per class are adopted.}
	\label{tab:resnet_cifar_un_schemes}
    \begin{tabular}{ccccccc}  
		\toprule
		Methods  &  Res-50\%  &  Res-70\% & Res-90\% & Res-95\% \\
		\midrule
		L1-norm &  $92.32_{\pm 0.00}$ & $81.82_{\pm 0.00}$ & $10.00_{\pm 0.00}$ & $10.00_{\pm 0.00}$  \\
		BP & $93.45_{\pm 0.07}$ & $91.87_{\pm 0.23}$ & $74.74_{\pm 1.01}$ &  $44.26_{\pm 1.84}$ \\
		FitNet & $93.60_{\pm 0.10}$ & $92.91_{\pm 0.17}$ & $84.12_{\pm 1.17}$ & $\mathbf{67.53_{\pm 1.75}}$ \\
		\midrule
		Ours-NC& $93.36_{\pm 0.04}$ & $93.07_{\pm 0.11}$ & $83.39_{\pm 0.82}$ & $55.15_{\pm 1.89}$ \\
		Ours & $\mathbf{93.57_{\pm 0.07}}$ & $\mathbf{93.23_{\pm 0.14}}$ & $\mathbf{86.57}_{\pm 0.69}$ & $\mathbf{68.07_{\pm 1.73}}$ \\
		Ours-S & $\mathbf{93.58_{\pm 0.04}}$ & $\mathbf{93.24_{\pm 0.10}}$ & $86.38_{\pm 0.79}$  &  $64.55_{\pm 1.63}$ \\
		\bottomrule
	\end{tabular}
	\end{minipage}}
\end{table*}

\subsection{Implementation Details}
For the CIFAR-10 experiments, we adopt the implementation from torchvision\footnote{https://github.com/pytorch/vision/blob/master/torchvision/models/} and follow the standard way~\cite{simonyan2014very,he2016deep} in pretraining the model. For VGG-16 with BN and ResNet-34 on ImageNet-IlSVRC12, we adopt the checkpoint from the official release of torchvision\footnote{https://pytorch.org/docs/stable/torchvision/models.html}. 
Similar to~\cite{li2019knowledge}, we do not adopt data augmentation so as to better simulate the few shot setting. The combination of our methods with various data augmentation skills are discussed later.

In terms of baselines, we adopt the implementation\footnote{https://github.com/Eric-mingjie/rethinking-network-pruning/tree/master/imagenet/l1-norm-pruning} from~\cite{liu2018rethinking} for 1) L1-Norm.
Based on the pruned models by 1), we perform fine-tuning with back-propagation by minimizing the cross entropy or the FitNet loss, denoted as 2) BP and 3) FitNet respectively.
For 4) ThiNet, our implementation is based on the published code\footnote{https://github.com/Roll920/ThiNet}.
For 5) CP, we re-implement the paper based on its TensorFlow version\footnote{https://github.com/Tencent/PocketFlow} and reproduce the results in Table 1 of the paper. 
Note that we do not consider tricks such as the residual compensation and the 3C enhancement since they are not the main focus of this paper, despite that we found residual compensation can lead to slight improvement for both CP and our methods.
For both ThiNet and CP, we use all feature map patches for regression instead of sampling a subset of them, the latter of which lead to a significant drop of accuracy when only limited training instances are available.

We adopt the ADAM optimizer for these methods, and adjust the learning rate within [1e-5, 1e-3] to obtain proper performance. Each layer is optimized for 3,000 iterations, where the sparsity ratio linearly increases within the first 1,000 iterations. After layer-wise training, we further fine-tune the network for a few more epochs with back-propagation.

\subsection{Pruning Schemes}
The detailed pruning schemes of VGG-16 network for structured pruning are as follows: For VGG-A, 50\% channels are removed for all blocks except the 3-rd block. For VGG-B, 10\% more channels are pruned based on VGG-A, and 20\% channels are pruned in the 3-rd block. For VGG-C, we prune 75\% channels in the first four blocks, and 60\% channels in the last block.

The reduction of network parameters as well as computational FLOPs of VGG-16, ResNet-56 and ResNet-34 for structured pruning are presented in Table~\ref{tab:structure_flops_vgg_cifar}, Table~\ref{tab:structure_flops_resnet_cifar} and Table~\ref{tab:structure_flops_resnet_ilsvrc} respectively.

For unstructured pruning, the model size can be directly calculated by the sparsity $r$ in theory. However, in practice the irregular sparsity induced by unstructured pruning may not lead to speedup under prevalent computational frameworks. The gain of unstructured sparsity may rely on some specially designed hard-wares.

\begin{table}[h]
	\footnotesize
	\caption{Structured pruning schemes of VGG-16 on CIFAR-10.}
	\label{tab:structure_flops_vgg_cifar}
	\centering
	\begin{tabular}{ccccc}  
		\toprule
		Schemes &  Params (M)  & P$\downarrow(\%)$ & FLOPs (G) & F$\downarrow(\%)$ \\
		\midrule
		Orig. & 14.99 & - & 0.314 & - \\
		VGG-50\% & 4.53 & 69.78 & 0.082 & 73.95 \\
		VGG-A & 6.11 & 59.26 & 0.208 & 33.76 \\		
		VGG-B & 4.37 & 70.83 & 0.137 & 56.37 \\
		VGG-C & 2.92 & 80.55 & 0.061 & 80.45 \\
		\bottomrule
	\end{tabular}
\end{table}

\begin{table}[h]
	\footnotesize
	\caption{Structured pruning schemes of ResNet-56 on CIFAR-10.}
	\label{tab:structure_flops_resnet_cifar}
	\centering
	\begin{tabular}{ccccc}  
		\toprule
		Schemes &  Params (M)  & P$\downarrow(\%)$ & FLOPs (G) & F$\downarrow(\%)$ \\
		\midrule
		Orig. & 0.85 & - & 0.127 & - \\
		Res-50\%  & 0.50 & 41.18 & 0.072 & 43.31 \\
		Res-60\% &  0.42 & 50.59 & 0.059 & 53.51 \\		
		Res-70\% & 0.35 & 58.82 & 0.048 & 62.20 \\
		Res-90\% & 0.21 & 75.29 & 0.028 & 77.95 \\
		\bottomrule
	\end{tabular}
\end{table}

\begin{table}[h!]
	\footnotesize
	\caption{Structured pruning schemes of ResNet-34 on ILSVRC-12.}
	\label{tab:structure_flops_resnet_ilsvrc}
	\centering
	\begin{tabular}{ccccc}  
		\toprule
		Schemes &  Params (M)  & P$\downarrow(\%)$ & FLOPs (G) & F$\downarrow(\%)$ \\
		\midrule
		Orig. & 21.80 & - & 3.68 & - \\
		Res-30\% & 19.71 & 9.59 & 2.97 & 19.15 \\
		Res-50\% & 18.33 & 15.91 & 2.51 & 31.47 \\		
		Res-70\% & 16.92 & 22.37 & 2.05 & 44.26 \\
		Res-70\%+ & 12.79 & 41.32 & 1.85 & 49.78 \\
		\bottomrule
	\end{tabular}
\end{table}

\subsection{More Results}
\subsubsection{Structured Pruning}
We also conduct structured pruning with ResNet-56 on CIFAR-10.
Following a similar pattern in previous experiments, we first fix $K$ and evaluate cross distillation with various pruning schemes, and then show how cross distillation responds to different $K$ shots training samples.
The results are in Table~\ref{tab:resnet_cifar_K_shots} and Table~\ref{tab:resnet_cifar_schemes} respectively. 
Our methods again outperform the rest approaches, and fewer training instances or higher sparsities also enjoy to larger advantages.

\begin{table}
	\centering
	\footnotesize
	\setlength{\tabcolsep}{2 pt}{
	\caption{The top-1 accuracy (\%) of different structured pruning schemes with Resnet-56 on CIFAR-10. 5 samples per class are used.}
	\label{tab:resnet_cifar_schemes}
	\begin{tabular}{lllllll}  
		\toprule
		Methods  &  Res-50\%  &  Res-60\% & Res-70\% & Res-90\% \\
		\midrule
		L1-norm &  $80.43_{\pm 0.00}$ & $50.55_{\pm 0.00}$ & $30.50_{\pm 0.00}$ & $14.70_{\pm 0.00}$  \\
		BP & $87.41_{\pm 0.98}$ & $80.85_{\pm 1.23}$ & $72.23_{\pm 1.81}$ &  $32.03_{\pm 2.23}$ \\
		FSKD & $87.53_{\pm 1.06}$ & $82.50_{\pm 0.95}$ & $70.93_{\pm 1.57}$ & $30.31_{\pm 1.76}$ \\
		FitNet & $89.43_{\pm 1.60}$ & $86.03_{\pm 1.96}$ & $\mathbf{81.90_{\pm 2.01}}$ & $51.15_{\pm 2.60}$ \\
		ThiNet & $89.54_{\pm 0.84}$ & $85.73_{\pm 0.97}$ & $79.75_{\pm 1.34}$ & $38.64_{\pm 1.78}$\\
		CP  & $89.39_{\pm 0.80}$ & $86.01_{\pm 0.84}$ & $80.20_{\pm 1.26}$ & $\mathbf{52.17_{\pm 1.43}}$ & \\
		\midrule
		Ours-NC& $89.51_{\pm 1.13}$ & $85.56_{\pm 1.32}$ & $80.53_{\pm 1.45}$ & $50.98_{\pm 1.60}$ \\
		Ours & $89.93_{\pm 1.21}$ & $\mathbf{86.80_{\pm 1.12}}$ & $79.67_{\pm 1.56}$ & $50.67_{\pm 1.98}$ \\ 
		Ours-S & $\mathbf{90.14_{\pm 1.19}}$ & $86.58_{\pm 0.97}$ & $\mathbf{81.86_{\pm 1.43}}$  &  $\mathbf{52.05_{\pm 1.75}}$ \\
		\bottomrule
	\end{tabular}}
\end{table}

\subsubsection{Unstructured Pruning}
For unstructured pruning, we also conduct VGG-16 and ResNet-56 on the CIFAR-10 dataset.
With similar patterns to previous experiments, the results of VGG-16, ResNet-56 are shown in Table~\ref{tab:vgg_cifar_un_schemes} and Table~\ref{tab:vgg_cifar_un_K_shots}, Table~\ref{tab:resnet_cifar_un_schemes} and Table~\ref{tab:resnet_cifar_un_K_shot} respectively.

\subsubsection{Quantization} 
To demonstarte the effectiveness of cross distillation for network quantization, we take VGG-16 and ResNet-56 on CIFAR-10 for illustration. We choose the regularization $\mathcal{R}(\m W^S)=\Pi_Q(g(\m W^S))$ to be the projection function, and $g(\m x)=\frac{\m x-\min(\m x)}{\max(\m x)-\min(\m x)}$ as the linear normalization function. 
Before applying the transformation, we first truncate $\m W^S$ by the three-sigma rule so as to avoid outliers that may lead to poorly distributed quantization points.
For activation quantization, we adopt the widely used clipping method to bound activations between $[0, 1]$.
We vary the training size between 1-shot and 5-shot, and the results are shown in Table~\ref{tab:quantization}.

\subsection{Other Analysis}
\subsubsection{Cross Distillation with Data Augmentation}

As we consider the setting of few shot network compression, a natural question is how data augmentation can help to alleviate the shortage of training data.
Here we combine cross distillation with 1) randomly crop/rotate the input image~(Rand.), which is the widely used data augmentation technique; 2) Gaussian noise~(Gauss.) $\epsilon\sim \mathcal{N}(0, 0.2\max (\m h^S))$ over the hidden feature map $\m h^S$; and 3) Mixup~\cite{zhang2017mixup}~(Mixup), a pair-wise interpolation method over the training samples.
We compare to Mixup since the idea of cross distillation resembles Mixup in a way that both methods apply convex combinations but over different levels. Cross distillation applies convex combinations on the loss level (Ours) and feature map level (Ours-S) between the student and teacher network, while Mixup applies those on the pairs of training samples.
We adopt structured pruning on CIFAR-10 and use VGG-50\% as the pruning scheme. 

From Table~\ref{tab:vgg_cifar_augmentation}, we find that while 1) Random noise and 2) Gaussian noise give a slight improvement in the few shot setting, Mixup can significantly boost the performance, especially when $K$ is small. Moreover, Ours and Ours-S benefit more from Mixup comparing to Ours-NC, which shows that the convex combination on the loss level and feature map level can be better combined with the interpolation of input pairs for few shot training.


\subsubsection{Layers of Cross Connection}
Finally, in order to further investigate which layers benefit most from cross distillation, we conduct ablation studies on the positions of cross connections.
We take Ours-S for illustration and adopt VGG-50\% for structured pruning with $K=10$.
Note that cross connections are assigned at the end of VGG blocks, e.g., C2.2 denotes conv2.2 of the VGG network. 
The results are shown in Table~\ref{tab:cross_position}. 
It can be observed that the crossing points at deeper layers tend to bring more improvement, which is consistent with the findings in Figure~\ref{fig:regu_loss} and Figure~\ref{fig:estimation_inconsistency} that larger error gaps and lower validation losses occur in deeper layers of the network.

\begin{table*}
	\caption{The accuracy (\%) on unstructured pruning with VGG-16 on CIFAR-10 with different training sizes. We choose VGG-50\% as the pruning scheme, and the accuracy of the original model is 93.51\%.}
	\label{tab:vgg_cifar_un_K_shots}
	\centering
	\footnotesize
	\begin{tabular}{ccccccc}  
		\toprule
		Methods  &  1  &  2 & 3 & 5 & 10 & 50 \\
		\midrule
		L1-norm & $15.06_{\pm 0.00}$ & $15.06_{\pm 0.00}$ & $15.06_{\pm 0.00}$ & $15.06_{\pm 0.00}$ & $15.06_{\pm 0.00}$ & $15.06_{\pm 0.00}$\\
		BP & $47.20_{\pm 1.43}$ & $55.76_{\pm 1.29}$ & $65.71_{\pm 1.30}$ & $71.38_{\pm 1.16}$ & $79.95_{\pm 0.89}$ & $84.95_{\pm 0.63}$\\
		FitNet & $69.30_{\pm 1.98}$ & $78.30_{\pm 1.77}$ & $81.60_{\pm 1.85}$ & $83.50_{\pm 1.73}$ & $86.59_{\pm 1.40}$ & $88.43_{\pm 1.05}$\\
		\midrule
		Ours-NC& $78.71_{\pm 1.24}$ & $86.74_{\pm 1.31}$ & $88.00_{\pm 1.07}$ & $89.03_{\pm 0.82}$ & $90.26_{\pm 0.83}$ & $91.12_{\pm 0.51}$\\
		Ours & $\mathbf{83.32_{\pm 1.35}}$ & $88.26_{\pm 1.10}$ & $\mathbf{90.07_{\pm 0.82}}$ & $\mathbf{90.62_{\pm 0.90}}$ & $\mathbf{91.41_{\pm 0.68}}$ & $91.48_{\pm 0.40}$ \\
		Ours-S & $81.23_{\pm 1.29}$ & $\mathbf{88.43_{\pm 1.14}}$ & $89.54_{\pm 0.98}$ & $90.36_{\pm 0.89}$ & $91.09_{\pm 0.79}$ & $\mathbf{91.91_{\pm 0.37}}$  \\
		\bottomrule
	\end{tabular}
\end{table*}

\begin{table*}
	\caption{The accuracy (\%) on unstructured pruning with ResNet-56 on CIFAR-10 with different training sizes. The pruning scheme is Res-90\% and the original accuracy is 93.32\%.}
	\label{tab:resnet_cifar_un_K_shot}
	\centering
	\footnotesize
	\begin{tabular}{ccccccc}  
		\toprule
		Methods  &  1  &  2  &  3  &  5  &  10  &  50 \\
		\midrule
		L1-norm & $10.00_{\pm 0.00}$ & $10.00_{\pm 0.00}$ & $10.00_{\pm 0.00}$ & $10.00_{\pm 0.00}$ & $10.00_{\pm 0.00}$ & $10.00_{\pm 0.00}$\\
		BP & $56.69_{\pm 1.70}$ & $62.53_{\pm 1.63}$ & $67.46_{\pm 1.32}$ & $74.74_{\pm 1.01}$ & $76.85_{\pm 1.12}$ & $82.03_{\pm 0.71}$\\
		FitNet & $70.65_{\pm 1.28}$ & $79.46_{\pm 1.38}$ & $82.17_{\pm 1.15}$ & $84.12_{\pm 1.17}$ & $85.45_{\pm 0.88}$ & $87.23_{\pm 0.45}$\\
		\midrule
		Ours-NC& $72.71_{\pm 1.38}$ & $80.20_{\pm 0.97}$ & $82.55_{\pm 0.73}$ & $83.39_{\pm 0.82}$ & $86.26_{\pm 0.53}$ & $87.68_{\pm 0.39}$\\
		Ours & $75.73_{\pm 1.28}$ & $\mathbf{84.24_{\pm 1,03}}$ & $\mathbf{85.06_{\pm 0.81}}$ &  $\mathbf{86.57_{\pm 0.70}}$ & $87.03_{\pm 0.52}$ & $87.82_{\pm 0.28}$ \\
		Ours-S & $\mathbf{80.59_{\pm 1.15}}$ & $82.67_{\pm 1.25}$ & $84.43_{\pm 0.97}$ & $\mathbf{86.38_{\pm 0.79}}$ & $\mathbf{87.85_{\pm 0.60}}$ & $\mathbf{88.03_{\pm 0.31}}$  \\
		\bottomrule
	\end{tabular}
\end{table*}

\begin{table*}[]
	\centering
	\footnotesize
	\setlength{\tabcolsep}{2 pt}{
	\begin{minipage}{0.48\textwidth}
	\caption{The accuracy(\%) of quantization with VGG-16 and ResNet-56 on CIFAR-10. "WxAy" means x-bit weight quantization and y-bit activation quantization.}
    \label{tab:quantization}
    \begin{tabular}{c|c|c|c|c}
    \toprule
    \multirow{2}{*}{VGG-16} &
    \multicolumn{2}{c|}{K=1} &
    \multicolumn{2}{c}{K=5} \\
    \cline{2-5}
      & W2A3 & W2A4 & W2A3 & W2A4  \\
    \hline
    Ours-NC & $53.24_{\pm 2.31}$ & $88.93_{\pm 0.41}$ & $80.89_{\pm 0.98}$ & $91.23_{\pm 0.18}$ \\
    \hline
    Ours & $\mathbf{60.36_{\pm 2.86}}$ & $\mathbf{89.12_{\pm 0.38}}$ & $\mathbf{81.20_{\pm 1.26}}$ & $\mathbf{91.35_{\pm 0.29}}$ \\
    \hline\hline
    \multirow{2}{*}{ResNet-56} &
    \multicolumn{2}{c|}{K=1} &
    \multicolumn{2}{c}{K=5} \\
    \cline{2-5}
      & W2A32 & W4A32 & W2A32 & W4A32  \\
    \hline
    Ours-NC & $72.48_{\pm 1.94}$ & $85.75_{\pm 0.96}$ & $84.67_{\pm 1.89}$ & $91.09_{\pm 0.37}$ \\
    \hline
    Ours & $\mathbf{80.92_{\pm 2.23}}$ & $\mathbf{90.42_{\pm 0.53}}$ & $\mathbf{86.11_{\pm 1.97}}$ & $\mathbf{91.23_{\pm 0.45}}$ \\
    \bottomrule
    \end{tabular}
	\end{minipage}\hspace{4ex}
	\begin{minipage}{0.48\textwidth}
	\caption{The accuracy(\%) on quantization with ResNet-56 on CIFAR-10. "WxAy" means x-bit weight quantization and y-bit activation quantization.}
    \label{tab:vgg_cifar_augmentation}
    \begin{tabular}{cccccc}  
		\toprule
		Methods  &  1  &  2 & 5 & 10 \\
		\midrule
		Ours-NC& $65.57_{\pm 1.61} $ & $75.44_{\pm 1.69}$ & $81.20_{\pm 1.19}$ & $84.07_{\pm 0.83}$ \\
		Ours+Rand. & $64.60_{\pm 2.04}$   & $73.10_{\pm 1.85}$  &  $82.40_{\pm 1.50}$ & $84.31_{\pm 1.30}$  \\
		Ours+Gauss. & $62.38_{\pm 1.32}$ & $75.15_{\pm 1.64}$ & $81.58_{\pm 1.55}$ & $83.69_{\pm 1.18}$ \\
		Ours+Mixup & $68.40_{\pm 1.41}$ & $78.59_{\pm 1.35}$ & $81.05_{\pm 1.07}$ & $83.47_{\pm 0.98}$ \\
		\midrule
		Ours & $69.25_{\pm 1.39}$ & $80.65_{\pm 1.47}$ & $84.91_{\pm 0.98}$ & $86.61_{\pm 0.71}$ \\
		Ours+Rand. & $72.09_{\pm 1.66}$ & $81.39_{\pm 1.53}$  & $85.31_{\pm 1.13}$ & $86.34_{\pm 0.89}$ \\
		Ours+Gauss. & $73.10_{\pm 1.66}$ & $81.46_{\pm 1.40}$  & $85.07_{\pm 1.22}$ & $86.34_{\pm 1.07}$ \\
		Ours+Mixup & $\mathbf{79.97_{\pm 1.71}}$ & $\mathbf{84.37_{\pm 1.32}}$  & $86.01_{\pm 0.99}$ & $87.01_{\pm 0.81}$ \\		
		\midrule
		Ours-S & $68.53_{\pm 1.59}$ & $76.83_{\pm 1.43}$ & $82.74_{\pm 1.19}$ & $86.30_{\pm 0.79}$ \\
		Ours-S+Rand. & $69.43_{\pm 1.95}$ & $78.25_{\pm 1.79}$  & $84.71_{\pm 1.68}$ & $86.22_{\pm 1.46}$ \\
		Ours-S+Gauss. & $69.79_{\pm 1.98}$ & $77.74_{\pm 1.80}$  & $83.97_{\pm 1.34}$ & $86.00_{\pm 1.20}$ \\
		Ours-S+Mixup & $79.63_{\pm 1.45}$ & $83.63_{\pm 1.21}$  & $\mathbf{86.10_{\pm 1.07}}$ & $\mathbf{87.12_{\pm 0.76}}$ \\
		\bottomrule
	\end{tabular}
	\end{minipage}}
\end{table*}


\begin{table*}
	\caption{Soft cross distillation with structured pruning at different layers of VGG-16 on CIFAR-10. Double comb. denotes there are two crossing points and it holds similarly for triple comb.. As the baseline, Ours-S has accuracy of $86.30_{\pm 0.79}$. }
	\label{tab:cross_position}
	\centering
	\footnotesize
	\setlength{\tabcolsep}{2 pt}{
	\begin{tabular}{llllll}  
		\toprule
		Single Layer & C1.2  & C2.2  & C3.3 & C4.3 & C5.3 \\
		Accuracy & $83.98_{\pm 0.63}$ & $83.30_{\pm 65}$ & $83.78_{\pm 59}$ & $84.65_{\pm 70}$ & $\mathbf{84.97_{\pm 71}}$ \\
		\midrule
		Double Comb. & C1.2 + C2.2  & C2.2 + C3.3 & C3.3 + C4.3 & C4.3 + C5.3 &  \\
		Accuracy & $84.05_{\pm 0.72}$ & $83.91_{\pm 0.68}$ & $83.84_{\pm 0.67}$ & $\mathbf{86.12_{\pm 0.70}}$ &  \\
		\midrule
		Triple Comb. & C1.2 + C2.2 + C3.3  & C2.2 + C3.3 + C4.3 & C3.3 + C4.3 + C5.3 &  & \\
		Accuracy & $84.12_{\pm 0.59}$ & $84.27_{\pm 0.64}$ & $\mathbf{86.26_{\pm 0.63}}$ & &  \\
		\bottomrule
	\end{tabular}}
\end{table*}


\end{document}